\documentclass{article}
\usepackage[margin=1in]{geometry}

\usepackage[hyphens]{url}
\usepackage{graphicx}
\usepackage{natbib}
\usepackage{caption}
\usepackage{mathtools} 
\usepackage{amsfonts}
\usepackage{amssymb}
\usepackage{amsthm}

\usepackage{subcaption}
\usepackage{booktabs} 
\usepackage{cleveref} 

\usepackage{algorithm}
\usepackage{algorithmic}

\usepackage{tikz} 
\usepackage{xcolor}
\usepackage{microtype}
\usepackage[section]{placeins}

\newtheorem{remark}{Remark}
\newtheorem{example}{Example}
\newtheorem{proposition}{Proposition}

\newtheorem{corollary}{Corollary}

\title{Parallel Noising in Neural Markov Logic Networks}

\author{
    Peter~Jung\textsuperscript{1},
    Giuseppe~Marra\textsuperscript{2},
    Ond\v{r}ej~Ku\v{z}elka\textsuperscript{1}
    \\[0.6em]
    \textsuperscript{1}Czech Technical University, Prague, Czech Republic\\
    \textsuperscript{2}Department of Computer Science, KU Leuven, Leuven, Belgium
}

\begin{document}
\maketitle

\begin{abstract}
 Neural Markov Logic Networks (NMLNs) are a flexible neurosymbolic relational model.
Previous work has shown that, although NMLNs achieve strong performance as generative models for small relational structures, they underperform diffusion-based generative graph models on larger structures. In this paper, we strengthen NMLNs along two main dimensions: (i) we increase the expressive capacity of their potential functions using graph neural networks, and (ii) we develop a new training and inference algorithm inspired by parallel-tempering Markov chain Monte Carlo methods, which we name \textit{parallel noising}. Together, these enhancements enable NMLNs to attain strong performance in graph generation relative to general diffusion-based generative graph models. Furthermore, they allow NMLNs to match the performance of specialized text-based recurrent models when generating small molecular structures.
\end{abstract}

\section{Introduction}

Neural Markov Logic Networks (NMLNs) are a flexible neurosymbolic relational model that combines the maximum-entropy semantics of Markov logic \citep{richardson2006markov} with learnable (neural) relational potentials \citep{marra2021neural,jung2024quantified}.
A persistent obstacle for using NMLNs in practice is \emph{inference}: likelihood training and downstream tasks require sampling from high-dimensional, multi-modal distributions over possible worlds, yet classical single-chain Gibbs sampling mixes poorly and is brittle in the presence of near-deterministic constraints \citep{marra2021neural}.

A natural first attempt to improve mixing is \emph{parallel tempering} (also known as \emph{replica exchange} \citep{swendsen1986replica,geyer1991markov,hukushima1996exchange}), which runs multiple chains targeting related ``flattened'' distributions and swaps their states so that ``hot'' chains can help the ``cold'' chain cross energy barriers.
In practice, however, as also confirmed by our initial experiments, temperature ladders can be brittle for NMLNs: the effective energy scale depends on model parameters and can drift substantially during training, so a ladder that yields good swap acceptance early on may quickly become ineffective.

We therefore propose \emph{parallel noising}: a replica-exchange MCMC scheme for NMLNs that replaces temperature ladders with a ladder of \emph{data-space corruptions}.
Concretely, we fix Bernoulli corruption levels $0\leq \nu_N<\cdots<\nu_1\leq 1$ and define a ladder of intermediate targets indexed by these noise rates.
During inference, each level performs within-level MCMC updates and we \emph{regularly} attempt swaps of adjacent levels using the standard replica-exchange Metropolis ratio, evaluated under the (unnormalized) NMLN log-potentials at the two noise levels rather than under temperature-rescaled energies.
The lowest-noise chain targets the desired NMLN, while higher-noise chains are easier to explore and help the sampler cross energy barriers.
Parallel noising offers two advantages that we make precise:
(i) it is an \emph{exact} replica-exchange algorithm for an arbitrary ladder of target distributions, so the lowest-noise marginal is the desired target at stationarity (Proposition~\ref{prop:pn-correctness}); and
(ii) because the ladder is defined by a fixed corruption operator rather than an energy rescaling, it remains well-behaved under \emph{time-varying} targets (e.g., during training), where temperature schedules are difficult to tune (Appendix \ref{app:toy-boolean-tracking} (toy example)).

\paragraph{Contributions.}
\begin{itemize}
    \item \textbf{Parallel noising for NMLNs:} we introduce a corruption-defined replica-exchange sampler that stays well-behaved under time-varying targets during learning, and combine it with global GNN energies to form NMLN*.
    \item \textbf{Why it works:} we characterize swap acceptance via distributional overlap and derive simple, distribution-free bounds for Bernoulli noising schedules.
    \item \textbf{Empirical impact:} NMLN* substantially improves mixing and sample quality on molecular graph generation benchmarks compared to prior NMLNs and strong generative baselines (Sec.~\ref{sec:experiments}).
\end{itemize}

\paragraph{Organization.}
Section~\ref{sec:preliminaries} reviews replica-exchange MCMC and the NMLN learning objective.
Section~\ref{sec:nmln} defines NMLNs and the Bernoulli corruption operator used to build a noise ladder.
Section~\ref{sec:limitations} summarizes the main bottlenecks of existing NMLNs---limited expressiveness of fragment-based potentials and slow-mixing Gibbs inference---and motivates our two remedies.
Section~\ref{sec:pn} introduces NMLN* with global GNN energies and the parallel noising sampler.
Section~\ref{sec:pn-theory} develops theoretical properties and connects parallel noising to parallel tempering.
Section~\ref{sec:experiments} presents experiments.

\section{Related Work}\label{sec:related}
\paragraph{Replica-exchange MCMC and tempering.}
Replica exchange (parallel tempering) couples chains targeting a ladder of related distributions and proposes swap moves so that exploratory (``hot'') chains help the target (``cold'') chain traverse energy barriers \citep{swendsen1986replica,geyer1991markov,hukushima1996exchange}.
Most practical variants define the ladder by \emph{temperature scaling}, and performance depends on maintaining sufficient overlap between adjacent temperatures.
For discrete energy-based models, this can be brittle because the effective energy scale can drift during learning, motivating adaptive and non-reversible tempering variants \citep{desjardins2010tempered,miasojedow2013adaptivept,syed2022non}.

\paragraph{Corruption/noising for discrete structures.}
Corruption operators are widely used to smooth discrete distributions and stabilize learning, including in denoising/diffusion models for graphs and molecules \citep{digress}.
In NMLNs, injecting noise during training was introduced as a pragmatic way to avoid near-deterministic constraints that can break Gibbs sampling \citep{marra2021neural}.
Our contribution is to use corruption to \emph{define} a replica-exchange ladder: PN remains an \emph{exact} sampler for the chosen intermediate targets, and these targets remain meaningful even when the underlying model evolves during training.

\paragraph{Neurosymbolic probabilistic logic and inference.}
NeSy systems include constraint- and fuzzy-logic approaches \citep{diligenti2017semantic,van2022analyzing,badreddine2022logic,xu2018semantic} and probabilistic-logic approaches \citep{manhaeve2018deepproblog,marra2020relational,li2023scallop,ahmed2022semantic,winters2022deepstochlog,van2025neurosymbolic}.
Within probabilistic NeSy, undirected energy-based models---including Markov logic networks \citep{richardson2006markov} and probabilistic soft logic models \citep{pryor2023neupsl}---are particularly relevant here; relational neural machines provide another closely related energy-based formulation \citep{marra2020relational}.
NMLNs \citep{marra2021neural,jung2024quantified} sit in this family but replace hand-specified first-order rules with neural potentials that can \emph{learn} soft constraints from data.

Exact inference in probabilistic NeSy systems such as DeepProbLog uses weighted model counting \citep{manhaeve2018deepproblog}, which is $\#\mathrm{P}$-hard in general \citep{abboud2020learning}. To scale inference and learning, practical systems use several distinct computational strategies: A*-like search over high-probability proofs \citep{manhaeve2021approximate}; provenance-semiring-based differentiable reasoning \citep{li2023scallop}; learned neural approximations to symbolic inference \citep{van2023nesi}; neural approximate model counting \citep{abboud2020learning}; sampling-based learning objectives \citep{verreet2024explain}; specialized gradient estimators for categorical random variables \citep{de2023differentiable}; hybrid approximate Bayesian inference, automated reasoning, and gradient estimation for sequential NeSy models \citep{desmet2025relational}; and general MCMC methods \citep{robert1999monte}, including Gibbs sampling in NMLNs \citep{marra2021neural}.
Our focus is on making MCMC-style inference effective even when the learned energy encodes highly deterministic structure.

\paragraph{Generative NeSy models.}
A subset of NeSy models extends these frameworks to generative tasks, e.g., via constrained deep generative models, including GANs \citep{di2020efficient,stoian2024realistic}, VAEs \citep{misino2022vael}, sequential neurosymbolic Markov models \citep{desmet2025relational}, or diffusion models \citep{huang2024symbolic,scassola2023zero}.
These methods span heterogeneous output domains, including tabular records, symbolic music, sequential latent states, and perceptual data. By contrast, NMLNs directly define joint distributions over complete symbolic worlds and can be used as generative models for graphs and relational databases. \citet{van2025neurosymbolic} use diffusion to model dependencies among symbolic concepts in perceptual NeSy pipelines, whereas NMLNs directly model complete symbolic worlds.

\section{Preliminaries}\label{sec:preliminaries}
We briefly review replica-exchange MCMC (parallel tempering) notation and highlight the acceptance mechanism that underlies PN.
The key idea is to run a \emph{ladder} of related MCMC chains, each targeting a related distribution, and allow state swaps so that exploratory chains help the target-level chain cross energy barriers.
A more detailed review (including an illustrative example) is in Appendix~\ref{sec:parallel_tempering_app}.

\subsection{Replica Exchange and Parallel Tempering}\label{sec:parallel_tempering}
Replica exchange runs $N$ Markov chains targeting a ladder of distributions $\{\pi_i\}_{i=1}^N$ and regularly proposes swaps between adjacent levels.
Writing $\pi_i(\omega)\propto \exp(\rho_i(\omega))$, the swap of states $(\omega_i,\omega_{i+1})$ is accepted with probability
\begin{equation}\label{eq:pt-swap}
\begin{aligned}
\alpha_{i,i+1}
&= \min\Big\{1,\exp\big(
\rho_i(\omega_{i+1})-\rho_i(\omega_i) \\
&\qquad\qquad\quad
+\rho_{i+1}(\omega_i)-\rho_{i+1}(\omega_{i+1})
\big)\Big\}.
\end{aligned}
\end{equation}

Overlaps between adjacent levels govern swap rates: when $\pi_i$ and $\pi_{i+1}$ concentrate on very different regions, swaps are rarely accepted and the ladder decouples.
Practical schedules therefore aim to keep adjacent targets sufficiently close (e.g., by geometric temperature spacing in PT or by gradual corruption increments in PN).

\section{Neural Markov Logic Networks}\label{sec:nmln}

\subsection{Model definition and learning}

In Neural Markov Logic Networks, we consider a function-free first-order logic language $\mathcal{L}$, which is built from a set of constants $\mathcal{C}_\mathcal{L}$ and predicates $\mathcal{R}_\mathcal{L} = \bigcup_i \mathcal{R}_i$, where $\mathcal{R}_i$ contains the predicates of arity i.
For $c_1, c_2, \dots, c_m \in \mathcal{C}_\mathcal{L}$ and $R\in\mathcal{R}_m$, we call $R(c_1, c_2, \dots, c_m)$ a \textit{ground atom}.
We define \textit{possible world} $\omega$ to be the pair $(\mathcal{C}, \mathcal{A})$, where $\mathcal{C} \subseteq \mathcal{C}_\mathcal{L}$, $\mathcal{A}$ is a subset of the set of all ground atoms that can be built from the constants in $\mathcal{C}$ and any relation in $\mathcal{R}_\mathcal{L}$. We define
$\Omega_\mathcal{L}$ to be the set of all possible worlds over $\mathcal{L}$.
Intuitively, a given possible world defines a set of $true$ facts one can state using the constants (entities) and the relations of the language~$\mathcal{L}$.

Let $\Phi(\omega;\mathbf{w}) : \Omega_\mathcal{L} \to \mathbb{R}$ be a parametric relational potential implemented by a relational neural network that maps a possible world $\omega$ to a scalar score.

Given a set of potential functions $\Phi_1$, $\dots$, $\Phi_m$, a \textit{neural Markov logic network (NMLN)} \citep{marra2021neural} is the parametric exponential-family distribution over possible worlds from a given $\Omega_{\mathcal{L}}$:
\[P(\omega) = \frac{1}{Z} \exp{\left(\sum_{i} \beta_i \Phi_i(\omega; \mathbf{w}_i) \right)},\]
where $\beta_i$ and $\mathbf{w}_i$  are parameters and $Z = \sum_{\omega \in \Omega_\mathcal{L}} \exp{\left(\sum_{i} \beta_i \Phi_i(\omega; \mathbf{w}_i) \right)}$ is the normalization constant (partition function).
Given a training set of possible worlds $\mathcal{Y}$, NMLNs can be learned by maximizing the following log-likelihood by some gradient-based method \citep{marra2021neural}:
\begin{align}
\label{eq:max_likelihood}
\max_{\mathbf{w}_i, \beta_i} \left\{ \sum_{\hat{\omega} \in \mathcal{Y}} \sum_{i=1}^m \beta_i \Phi_i(\widehat \omega; \mathbf{w}_i) - \log Z \right\}.
\end{align}

The gradient of the log-likelihood function of NMLNs, given the data $\widehat{\omega}$ (here $\widehat{\omega}$ is an example from the training data, represented as a possible world) takes the form:

\begin{align}
\label{eq:derivative_w}
\frac{\partial \log P(\widehat \omega)}{\partial w_{i,j}} & =  \beta_i  \left( \frac{\partial \Phi_i(\widehat \omega; \mathbf{w}_i)}{\partial w_{i,j}}   -  \mathbb{E}_{ \omega \sim  P}\!\left[\frac{\partial \Phi_i(\omega;\mathbf{w}_i)}{\partial w_{i,j}}\right] \right)\\
\label{eq:derivative_beta}
\frac{\partial \log P(\widehat \omega)}{\partial \beta_i} & = \Phi_i(\widehat \omega; \mathbf{w}_i)  -  \mathbb{E}_{ \omega \sim  P}\!\left[\Phi_i(\omega; \mathbf{w}_i)\right]
\end{align}

Hence, what we need to be able to compute in order to perform the gradient-based learning are the four types of quantities: $\Phi_i(\widehat \omega; \mathbf{w}_i)$, $\frac{\partial \Phi_i(\widehat \omega; \mathbf{w}_i)}{\partial w_{i,j}}$, $\mathbb{E}_{ \omega \sim  P}\!\left[\frac{\partial \Phi_i(\omega;\mathbf{w}_i)}{\partial w_{i,j}}\right]$, and $\mathbb{E}_{ \omega \sim  P}[\Phi_i(\omega; \mathbf{w}_i)]$. The first one can be computed by evaluating the potential function $\Phi_i$ on the training example $\widehat{\omega}$, the second can also be obtained jointly for all $j$ using the standard backpropagation algorithm. The remaining two are more difficult to compute because they involve expectation over samples from the distribution given by the current weights. We can use sampling to approximate these expectations. In particular,  Gibbs Sampling is used in the original version of NMLNs, where it is used not only for computing the gradients but also for predictions---both for marginal and conditional inference (i.e., computing marginal and conditional probabilities) and for sampling in generative settings.

A significant drawback of Gibbs sampling is that it often requires a large number of steps before converging to the target distribution \citep{casella_book}.

\subsection{Bernoulli corruption and corruption ladders}\label{sec:noising}

A practical difficulty for MCMC in relational models is the prevalence of near-deterministic structure: hard or almost-hard constraints can create extremely sharp modes separated by regions of negligible probability mass. In such regimes, single-chain Gibbs sampling mixes poorly, because most local moves are rejected or trapped in a narrow basin.

We therefore introduce a \emph{Bernoulli corruption operator} on possible worlds. Fix a domain size and identify a possible world $\omega$ with the binary vector of its ground atoms in a fixed vocabulary; let $d$ denote the number of such atoms. For a noise parameter $\nu\in[0,1]$, define the bit-flip channel $K_\nu$ by independently flipping each atom with probability $\nu$. That is, for worlds $\omega,\omega'\in\Omega$,
\begin{equation}
K_\nu(\omega\to \omega') \;=\; \prod_{j=1}^{d}\Big[(1-\nu)\mathbf{1}\{\omega'_j=\omega_j\}
+\nu\,\mathbf{1}\{\omega'_j\neq\omega_j\}\Big].
\end{equation}
For any target distribution $\pi$ on $\Omega$ (e.g., the NMLN defined by learned potentials), the \emph{$\nu$-noised} distribution is the pushforward
\[
\pi_\nu(\omega') \;=\; \sum_{\omega\in\Omega}\pi(\omega)\,K_\nu(\omega\to \omega').
\]
\paragraph{Corruption for molecular graphs.}
The theoretical development treats $\omega$ as a binary vector of ground atoms. In the molecule experiments, each molecule is represented as a finite relational structure over a fixed vocabulary of predicates encoding (i) node labels (atom types) and (ii) edge labels (bond types). Concretely, we use unary predicates $\mathsf{Atom}_t(v)$ for atom type $t$ at node $v$ and binary predicates $\mathsf{Bond}_b(u,v)$ for bond type $b$ on an (undirected) edge $\{u,v\}$. Categorical attributes are encoded by one-hot groups of ground atoms.
Our corruption operator acts independently \emph{across} these groups but preserves the one-hot constraints \emph{within} each group: with probability $1-\nu$ we keep the current category, and with probability $\nu$ we resample a category uniformly from the remaining options.\footnote{Equivalently, this is a simple symmetric categorical channel; the bit-flip channel is recovered for truly binary predicates.} For undirected bonds we corrupt only pairs with $u{<}v$ and mirror the result to enforce symmetry.

We do not hard-enforce chemical validity (valence constraints, aromaticity, etc.) during MCMC; instead, validity is an \emph{evaluation} property (Section~\ref{sec:experiments}). In practice, the learned energy assigns low probability to invalid structures once such constraints are captured by the potentials.

Intuitively, increasing $\nu$ smooths the distribution and reduces the severity of narrow deterministic basins. In Section~\ref{sec:pn}, we use these noised distributions to construct a ladder of intermediate targets for replica exchange.

\section{Limitations of Existing NMLNs}\label{sec:limitations}
Section~\ref{sec:nmln} defined NMLNs and reviewed the Bernoulli corruption operator that has been used in prior work to smooth otherwise sharp energies.
While NMLNs are conceptually flexible, existing instantiations and learning pipelines still run into two bottlenecks that have limited their performance on larger structured domains:
(i) \emph{local potential families} whose effective receptive field is controlled by a fragment width parameter, and
(ii) \emph{sampling-based inference} that relies on Gibbs-style local updates and becomes brittle as the learned energy approaches deterministic constraints.
These limitations motivate two changes: GNN-based relational potentials and the replica-exchange sampler which we describe in Section~\ref{sec:pn}.

\subsection{Limited Expressive Power}\label{sec:limited_expressiveness}

In the original NMLN formulation \citep{marra2021neural} (and its recent extensions \citep{jung2024quantified}), potentials are built by aggregating scores of small induced fragments.
This also means that capturing long-range relational regularities (e.g., connectivity or path-based constraints in graphs) typically requires increasing the fragment width, quickly leading to $O(n^k)$ fragments and correspondingly expensive computation.

This intuition is formalized by \citet{jung2024quantified}: they show (Theorem~1 in \citep{jung2024quantified}) that NMLNs with potential functions of width $k$ can represent the same distributions as classical Markov Logic Networks \citep{richardson2006markov} whose formulas use at most $k$ variables, but \emph{without} quantifiers or constants.
The restriction to quantifier-free, bounded-variable logic is limiting in practice; for example, even simple global graph properties such as ``no isolated vertices'' cannot be captured with constant width (see \citep{jung2024quantified} for additional examples).

The work of \citet{jung2024quantified} also proposes more expressive variants and relates them to Quantified Markov Logic Networks \citep{DBLP:conf/kr/Gutierrez-Basulto18a}, which allow $\exists/\forall$ quantifiers in prenex normal form with at most $k$ variables.
However, as soon as the desired dependency spans longer paths, prenex restrictions again force $k$ to grow, and the model size and computation scale in the same prohibitive way as increasing fragment width.

These results motivate potential families with \emph{hierarchical} aggregation and an expanding receptive field at fixed parameter size.

In short, fragment-based potential families face a hard tradeoff between expressivity and scalability: capturing long-range relational structure typically forces the fragment width to grow, leading to an $O(n^k)$ blow-up in the number of induced fragments and in the cost of evaluating the energy.

\subsection{Inference (Sampling)}\label{sec:problems_with_inference}

The second limitation, already highlighted in the introduction, is inference.
Exact marginal inference and exact sampling in NMLNs are generally intractable under standard complexity assumptions.\footnote{This follows from the complexity of first-order model counting \citep{beame2015symmetric}, since WFOMC can be solved given access to an oracle for marginal inference in NMLNs. This follows the same reasoning as similar arguments done before for Markov Logic Networks in the literature, exploiting the fact that classical MLNs can be represented as NMLNs with suitably chosen potential functions.}
Accordingly, existing NMLN implementations rely on (variants of) Gibbs sampling in practice.\footnote{\citet{marra2021neural} also introduce variants of blocked Gibbs sampling to speed up inference, but their implementation still relies on Gibbs-style local updates.}

Unfortunately, Gibbs mixes poorly for the kinds of rugged, multi-modal energies that arise once the model starts to encode near-deterministic constraints.
In classical MLNs one can sometimes handle hard constraints using MC-SAT \citep{DBLP:conf/aaai/PoonD06}, but MC-SAT requires an explicit logical encoding of the deterministic clauses.
In NMLNs, constraints are encoded \emph{implicitly} by neural potentials, so this route is unavailable.

The original NMLN paper therefore injected noise during training: at the beginning of each epoch, each ground atom of a training world is flipped with probability $\pi_n$ \citep{marra2021neural}.
While this can make Gibbs sampling feasible, it couples modeling and inference: the model is optimized for a \emph{corrupted} data distribution, and truly deterministic structure can only be approximated.
As a result, the amount of corruption becomes a brittle hyperparameter and sample quality degrades when the task requires crisp constraints.

\subsection{Transient corruption as an auxiliary ladder}\label{sec:pn-overview}
This motivates using the corruption operator $K_\nu$ (Section~\ref{sec:noising}) only as a \emph{transient exploration device}:
we keep the (nearly) clean NMLN at the lowest noise level $\nu_N$ as the learning target---typically $\nu_N=0$, but it can also be a small positive noise---and use higher corruption levels to define a sequence of auxiliary intermediate distributions that are easier to mix.
In parallel, we alleviate the expressivity bottleneck with global GNN energies (Section~\ref{sec:gnn-potentials}).
Building on these two ingredients, \emph{parallel noising} runs a replica-exchange sampler across noise levels and swaps states, allowing highly corrupted chains to traverse modes while preserving the lowest-noise chain as an exact sampler of the desired target distribution (Section~\ref{sec:pn-sampler}).

In our experiments we set the target level to a small but nonzero noise, $\nu_N=10^{-3}$, so the lowest-noise chain samples from a lightly corrupted target. Setting $\nu_N=0$ recovers the fully clean NMLN when desired. More generally, one can choose $\nu_N$ to be any sufficiently small positive noise if the application calls for a slightly smoothed target.
Higher-noise levels are auxiliary and exist solely to accelerate exploration of the target chain.

\section{NMLN*: Global Potentials and Parallel Noising}\label{sec:pn}

We propose \emph{NMLN*}, an NMLN variant that addresses the two bottlenecks from Section~\ref{sec:limitations}: (i) limited expressive power of fragment-based potentials, and (ii) poor mixing of Gibbs-style inference on sharp, multi-modal energies.
NMLN* combines two orthogonal changes: \emph{global} graph neural network (GNN) energies (Section~\ref{sec:gnn-potentials}) and a corruption-defined replica-exchange sampler that we call \emph{parallel noising} (Sections~\ref{sec:pn-sampler}--\ref{sec:pn-schedule}).

\subsection{Global potentials via GNNs}\label{sec:gnn-potentials}
Prior NMLN variants commonly rely on \emph{local} fragment-based potentials (e.g., DeepSet-style aggregation), which can miss global structural signals needed for larger relational structures (and larger molecules).
We therefore instantiate $\Phi(\omega;\mathbf{w})$ as a \emph{global} GNN energy: we convert a world $\omega$ into a labeled graph whose node/edge labels correspond to unary/binary predicates (atom/bond types in molecules), apply message passing with relation-specific parameters (e.g., an R-GCN-like layer), pool node representations, and map the pooled vector to a scalar energy via an MLP \citep{rgcn}.
This change is independent of PN: any potential family can be used at each noise level, but sharper (more expressive) targets typically make mixing harder, increasing the benefit of replica exchange.
Additional architectural details are given in Appendix~\ref{app:exp-details}.

\subsection{Parallel noising: replica exchange over corruption levels}\label{sec:pn-sampler}
Parallel noising follows the standard replica-exchange template \citep{swendsen1986replica,geyer1991markov,hukushima1996exchange} (Section~\ref{sec:parallel_tempering}), but replaces temperature scaling with a ladder indexed by corruption levels.

\paragraph{Noise ladder and level-wise targets.}
Fix noise levels $0\le \nu_N < \nu_{N-1} < \cdots < \nu_1 \le 1$, where $\nu_N$ is the smallest noise (the \emph{target} level).
For each level $i\in\{1,\dots,N\}$ we define a target distribution $\pi_i$ with an unnormalized log-density $\rho_i:\Omega\to\mathbb{R}$:
\[
\pi_i(\omega) \;=\; \frac{\exp(\rho_i(\omega))}{Z_i},\qquad
Z_i=\sum_{\omega'\in\Omega}\exp(\rho_i(\omega')).
\]
In the learning setup of Section~\ref{sec:pn-training}, each $\pi_i$ is an NMLN with its own parameters $\theta_i$ trained on data corrupted at rate $\nu_i$; we then write $\rho_i(\omega)=\rho_{\theta_i}(\omega)$.
The sampling algorithm itself only requires that we can evaluate each $\rho_i(\omega)$ up to an additive constant.

\paragraph{Within-level updates and swaps.}
Parallel noising targets the product distribution on $\Omega^N$,
\[
\Pi(\omega_1,\dots,\omega_N)\;=\;\prod_{i=1}^N \pi_i(\omega_i),
\]
and alternates:
(i) \emph{within-level} updates using any Markov kernels $K_i$ that leave $\pi_i$ invariant (e.g., blocked Gibbs / Metropolis-within-Gibbs on $\omega_i$),
and (ii) \emph{swap} proposals between adjacent levels.
A swap proposal exchanges $(\omega_i,\omega_{i+1})\mapsto(\omega_{i+1},\omega_i)$ and is accepted with the Metropolis probability
\begin{equation}
\label{eq:pn-swap}
\begin{aligned}
\alpha_{i,i+1}(\omega_i,\omega_{i+1})
&=\min\left\{1,\;\frac{\pi_i(\omega_{i+1})}{\pi_i(\omega_i)}\cdot\frac{\pi_{i+1}(\omega_i)}{\pi_{i+1}(\omega_{i+1})}\right\}\\
&=\min\left\{1,\;\exp(\Delta_{i,i+1})\right\},
\end{aligned}
\end{equation}
where
\[
\Delta_{i,i+1}
=
\rho_i(\omega_{i+1})-\rho_i(\omega_i)+\rho_{i+1}(\omega_i)-\rho_{i+1}(\omega_{i+1}).
\]
This is the standard replica-exchange ratio (Remark~\ref{remark1}), but it does not require defining or tuning temperatures.
Algorithm~\ref{alg:pn} summarizes one sweep.

\begin{algorithm}[t]
\caption{One sweep of Gibbs sampling with Parallel Noising (PN)}
\label{alg:pn}
\begin{algorithmic}[1]
\STATE \textbf{Input:} states $(\omega_1,\dots,\omega_N)$, potentials $(\rho_1,\dots,\rho_N)$, within-level kernels $(K_1,\dots,K_N)$
\FOR{$i=1$ to $N$}
    \STATE $\omega_i \sim K_i(\omega_i,\cdot)$ \COMMENT{within-level update (e.g., blocked Gibbs)}
\ENDFOR
\FOR{$\text{phase}\in\{\text{even},\text{odd}\}$}
    \STATE $start \gets 1$ if phase is even else $2$
    \FOR{$i=start$ to $N-1$ step $2$}
        \STATE propose swap $(\omega_i,\omega_{i+1})\leftarrow(\omega_{i+1},\omega_i)$
        \STATE accept with probability $\alpha_{i,i+1}(\omega_i,\omega_{i+1})$ in \eqref{eq:pn-swap}
    \ENDFOR
\ENDFOR
\STATE \textbf{Output:} updated states $(\omega_1,\dots,\omega_N)$
\end{algorithmic}
\end{algorithm}

\subsection{Training objective and estimator}\label{sec:pn-training}
A central motivation for PN is to stabilize \emph{likelihood training} by improving the quality of model samples used in the negative phase.
The noise levels define auxiliary targets that help mixing; we fit one NMLN per level (no parameter sharing across levels) and ultimately care about the lowest-noise model.

\paragraph{Per-level likelihood objectives.}
Let $p_{\mathrm{data}}$ denote the empirical distribution over training worlds in $\Omega$.
For each level $i$, define the corrupted-data distribution $q_i := p_{\mathrm{data}} K_{\nu_i}$ (Section~\ref{sec:noising}).
We fit an NMLN at each level $i$ with parameters $\theta_i$ and potential $\rho_{\theta_i}$, inducing
\[
\pi_{\theta_i}(\omega) \;\propto\; \exp\big(\rho_{\theta_i}(\omega)\big).
\]
Training maximizes the sum of per-level log-likelihoods
\begin{equation}
\label{eq:multi-level-ml}
\max_{\theta_1,\dots,\theta_N}\;\sum_{i=1}^N\; \mathbb{E}_{\omega\sim q_i}\big[\rho_{\theta_i}(\omega)\big] \; - \; \log Z_{\theta_i},
\end{equation}
where $Z_{\theta_i}=\sum_{\omega'\in\Omega}\exp(\rho_{\theta_i}(\omega'))$.
The gradient for level $i$ has the usual positive/negative-phase form
\[
\nabla_{\theta_i} \ell_i
= \mathbb{E}_{\omega\sim q_i}\big[\nabla_{\theta_i} \rho_{\theta_i}(\omega)\big]
- \mathbb{E}_{\omega\sim \pi_{\theta_i}}\big[\nabla_{\theta_i} \rho_{\theta_i}(\omega)\big],
\]
where the negative-phase expectation is approximated by MCMC.

\paragraph{Persistent PN chains for the negative phase.}
We maintain $R$ persistent PN replicas (Section~\ref{sec:pn-sampler}), each containing states $(\omega^{(r)}_1,\dots,\omega^{(r)}_N)$.
At each SGD step we advance each replica by a small, fixed number $T$ of PN sweeps (Algorithm~\ref{alg:pn}), and we approximate the model expectation at level $i$ using the \emph{current} snapshot $\{\omega^{(r)}_i\}_{r=1}^R$.
We do not store or average over full chain histories; this yields a persistent stochastic gradient estimator analogous in spirit to persistent contrastive divergence \citep{tieleman2008training}, which maintains Markov chains across parameter updates, rather than ordinary contrastive divergence \citep{hinton2002training}, which reinitializes chains from data.
Algorithm~\ref{alg:train-pn} shows one SGD step.

\begin{algorithm}[t]
\caption{One SGD step for training per-level NMLNs with PN}\label{alg:train-pn}
\begin{algorithmic}[1]
\STATE \textbf{Input:} minibatch $\mathcal{B}$, noise levels $\nu_{1:N}$, parameters $\{\theta_i\}_{i=1}^N$, persistent states $\{\omega^{(r)}_i\}_{i,r}$, PN steps $T$
\FOR{$r=1$ to $R$}
    \FOR{$t=1$ to $T$}
        \STATE $(\omega^{(r)}_1,\dots,\omega^{(r)}_N)\gets \textsc{PN-Sweep}((\omega^{(r)}_1,\dots,\omega^{(r)}_N),\{\rho_{\theta_i}\}_{i=1}^N)$ \COMMENT{Alg.~\ref{alg:pn}}
    \ENDFOR
\ENDFOR
\FOR{$i=1$ to $N$}
    \STATE Sample corrupted minibatch $\widetilde{\mathcal{B}}_i$ by corrupting each $\omega\in\mathcal{B}$ with $K_{\nu_i}$
    \STATE $\widehat{\mathcal{L}}_i(\theta_i)\gets \frac{1}{|\widetilde{\mathcal{B}}_i|}\sum_{\tilde\omega\in\widetilde{\mathcal{B}}_i} \rho_{\theta_i}(\tilde\omega)\; -\; \frac{1}{R}\sum_{r=1}^R \rho_{\theta_i}(\omega^{(r)}_i)$
    \STATE Update $\theta_i \leftarrow \theta_i + \eta\, \nabla_{\theta_i} \widehat{\mathcal{L}}_i(\theta_i)$
\ENDFOR
\end{algorithmic}
\end{algorithm}

\subsection{Why corruption levels are attractive in changing models}\label{sec:pn-changing}
During likelihood training, the target distribution changes as parameters are updated.
In temperature-based parallel tempering, swap acceptance can deteriorate abruptly when the effective energy scale drifts, so temperature ladders may require frequent retuning.
In PN, the ladder is defined by the fixed corruption operator $K_{\nu}$ and fixed noise levels $\{\nu_i\}$; Section~\ref{sec:pn-theory} provides overlap and tracking guarantees that depend on the corruption increments rather than on a temperature scale.

\subsection{Choosing the noising schedule in practice}\label{sec:pn-schedule}
PN requires choosing the noise levels $\nu_1>\cdots>\nu_N$.
As with all replica-exchange methods, performance hinges on \emph{overlap} between adjacent targets: if $\pi_i$ and $\pi_{i+1}$ concentrate on disjoint regions, swaps are rarely accepted and the ladder decouples.
Section~\ref{sec:pn-theory} makes this connection precise by relating expected swap acceptance to an overlap quantity (Proposition~\ref{prop:swap-overlap}) and by providing a distribution-free sanity bound for Bernoulli noising (Proposition~\ref{prop:tv-noise}).
In practice, these results suggest using sufficiently fine spacing in $\nu$ so that adjacent levels retain non-trivial overlap, and validating this by monitoring swap acceptance rates.

\section{Theoretical Properties of Parallel Noising}\label{sec:pn-theory}
This section summarizes the key guarantees that guide practical PN design; proofs are in Appendix~\ref{app:theory} and Appendix~\ref{app:proofs}, except for Proposition~\ref{prop:swap-overlap}, which is a standard replica-exchange identity that we cite.

\paragraph{Correctness (exactness).}
PN is standard replica exchange on the product target $\Pi(\omega_{1:N})=\prod_{i=1}^N \pi_i(\omega_i)$: within-level kernels preserve each $\pi_i$, and Metropolis swaps preserve $\Pi$.
Consequently, in stationarity the marginal of the lowest-noise chain is \emph{exactly} the desired target $\pi_N$ (Proposition~\ref{prop:pn-correctness}).

\begin{proposition}[Correctness of PN]\label{prop:pn-correctness}
Assume that for each level $i$, the within-level kernel $K_i$ leaves $\pi_i$ invariant.
Let $K=\bigotimes_{i=1}^N K_i$ be the product update on $\Omega^N$ and let $S_{i,i+1}$ be the Metropolis swap kernel between adjacent levels $i$ and $i{+}1$ with acceptance probability as in Eq. \eqref{eq:pn-swap}.
Then:
\begin{enumerate}
    \item $K$ leaves the product distribution $\Pi(\omega_1,\dots,\omega_N)=\prod_{i=1}^N \pi_i(\omega_i)$ invariant;
    \item for each $i$, $S_{i,i+1}$ leaves $\Pi$ invariant;
    \item any composition of $K$ and swap kernels (e.g., Algorithm~\ref{alg:pn}) leaves $\Pi$ invariant; and
    \item in stationarity, the marginal distribution of the lowest-noise chain $\omega_N$ is exactly $\pi_N$.
\end{enumerate}
\end{proposition}

\noindent\emph{Proof.} See Appendix~\ref{app:proofs}.

\paragraph{Relation to temperature-based PT.}
Parallel tempering corresponds to replica exchange over a ladder of \emph{temperature-scaled} targets, typically of the form $\pi_i(\omega)\propto \exp(\beta_i\rho(\omega))$ for a base energy $\rho$ \citep{swendsen1986replica,geyer1991markov,hukushima1996exchange}.
Under this temperature-scaled choice, our swap acceptance in Eq. \eqref{eq:pn-swap} reduces to the standard PT ratio.
In contrast, PN constructs intermediate targets via Bernoulli corruption in data space; this ladder is generally not equivalent to temperature scaling, so we do \emph{not} claim that PT can be obtained from Bernoulli-corrupted ladders.
Rather, PT and PN are two different instantiations of the same replica-exchange template with different intermediate distributions.

\paragraph{Acceptance as distributional overlap.}
Expected swap acceptance between adjacent levels equals an overlap quantity: it is high exactly when adjacent targets have substantial overlap (Proposition~\ref{prop:swap-overlap}).
This motivates selecting $\{\nu_i\}$ so that adjacent acceptance rates remain non-trivial during training (we monitor this directly).

\begin{proposition}[Expected acceptance equals overlap]\label{prop:swap-overlap}
Let $\mu=\pi_i\otimes\pi_{i+1}$ and $\mu^{\mathrm{swap}}=\pi_{i+1}\otimes\pi_i$. If $(X,Y)\sim\mu$, then
\[
\mathbb{E}\big[\alpha_{i,i+1}(X,Y)\big]
=\int \min\{d\mu,d\mu^{\mathrm{swap}}\}
=1-\|\mu-\mu^{\mathrm{swap}}\|_{\mathrm{TV}},
\]
i.e., the mean swap acceptance equals an overlap coefficient between adjacent targets \citep{kofke2002acceptance}.
\end{proposition}

\paragraph{A simple worst-case sanity bound for Bernoulli noising.}
For Bernoulli corruption on $d$ ground atoms, the total variation distance between adjacent noise levels admits a distribution-free bound proportional to $d|\nu-\nu'|$ (Appendix Proposition~\ref{prop:tv-noise}), yielding a corresponding acceptance lower bound (Appendix Corollary~\ref{cor:acc-lb}).
While typically loose, it clarifies that schedules with large jumps in $\nu$ can destroy overlap on high-dimensional worlds.

\paragraph{Time-varying targets during learning.}
Because temperature ladders implicitly depend on an evolving energy scale, PT swap rates can degrade as parameters drift.
With PN the ladder is defined by a fixed corruption operator. We show in the appendix, on a toy example, that PN may continue to perform well even in situations where PT already degrades.

\section{Experiments}\label{sec:experiments}
We evaluate NMLN* (GNN potentials + PN) on four molecular benchmarks (ChEMBL, QM9, ZINC250k, MOSES).
We report results only on \textbf{ChEMBL} and defer the full suite, ablations, and implementation details to Appendix~\ref{app:exp-details}.

\paragraph{Protocol and metric.}
Following \citet{jung2024quantified}, we generate 2M samples per method and report \emph{recall curves}: among the $t$ most frequently generated unique molecules, how many occur in the test set (frequency as a proxy for model probability).

\begin{figure*}[ht]
    \centering

    \begin{subfigure}[b]{0.9\textwidth}  
        \centering
        \includegraphics[width=0.4\textwidth]{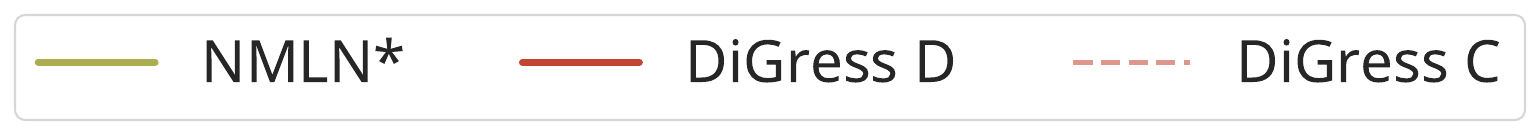}
    \end{subfigure}
    \vspace{0.2cm}
    \begin{subfigure}[b]{0.33\textwidth}
        \centering
        \includegraphics[width=\columnwidth]{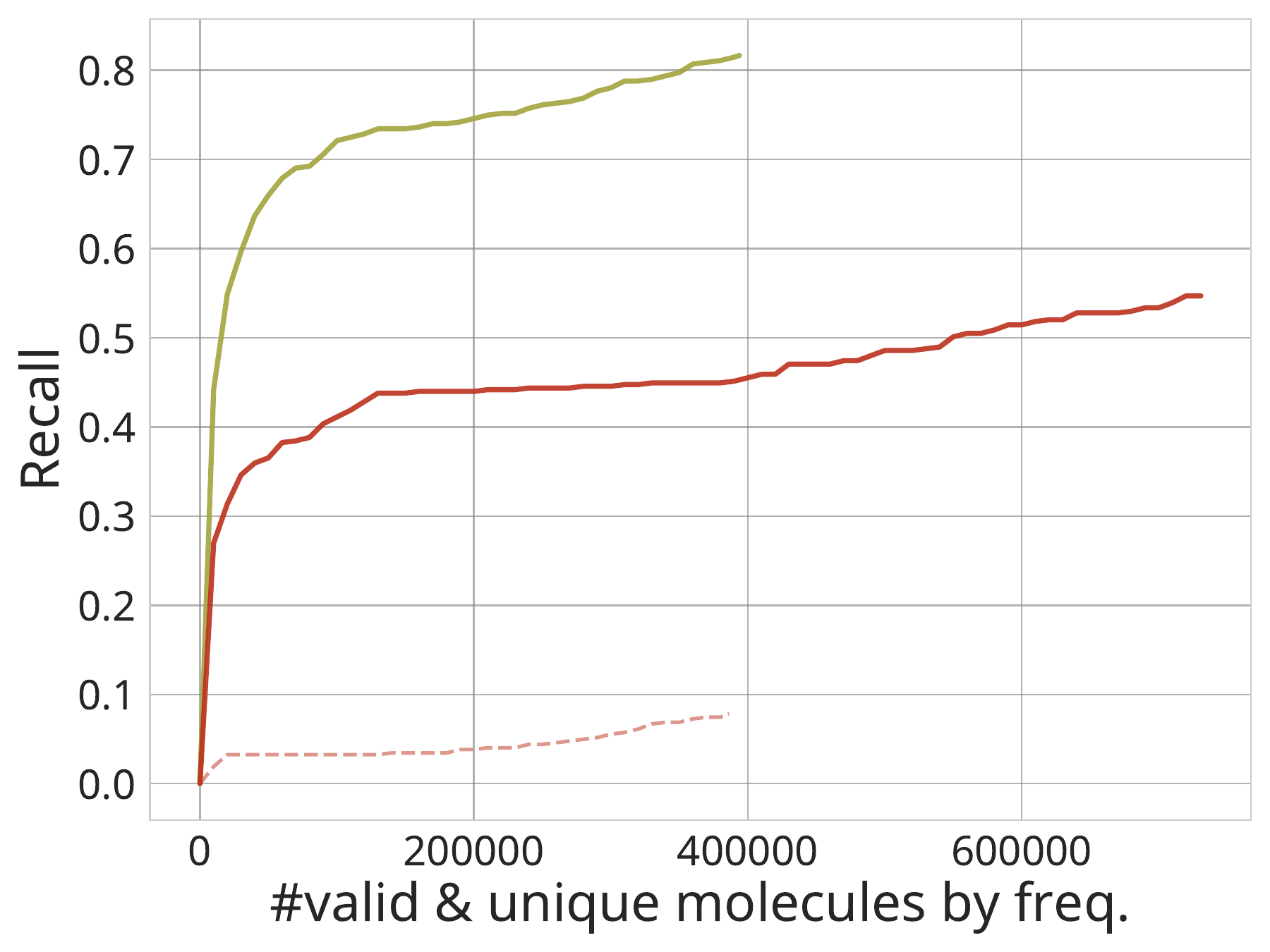}
        \caption{ChEMBL 10}
    \end{subfigure}
    \begin{subfigure}[b]{0.33\textwidth}
        \centering
        \includegraphics[width=\columnwidth]{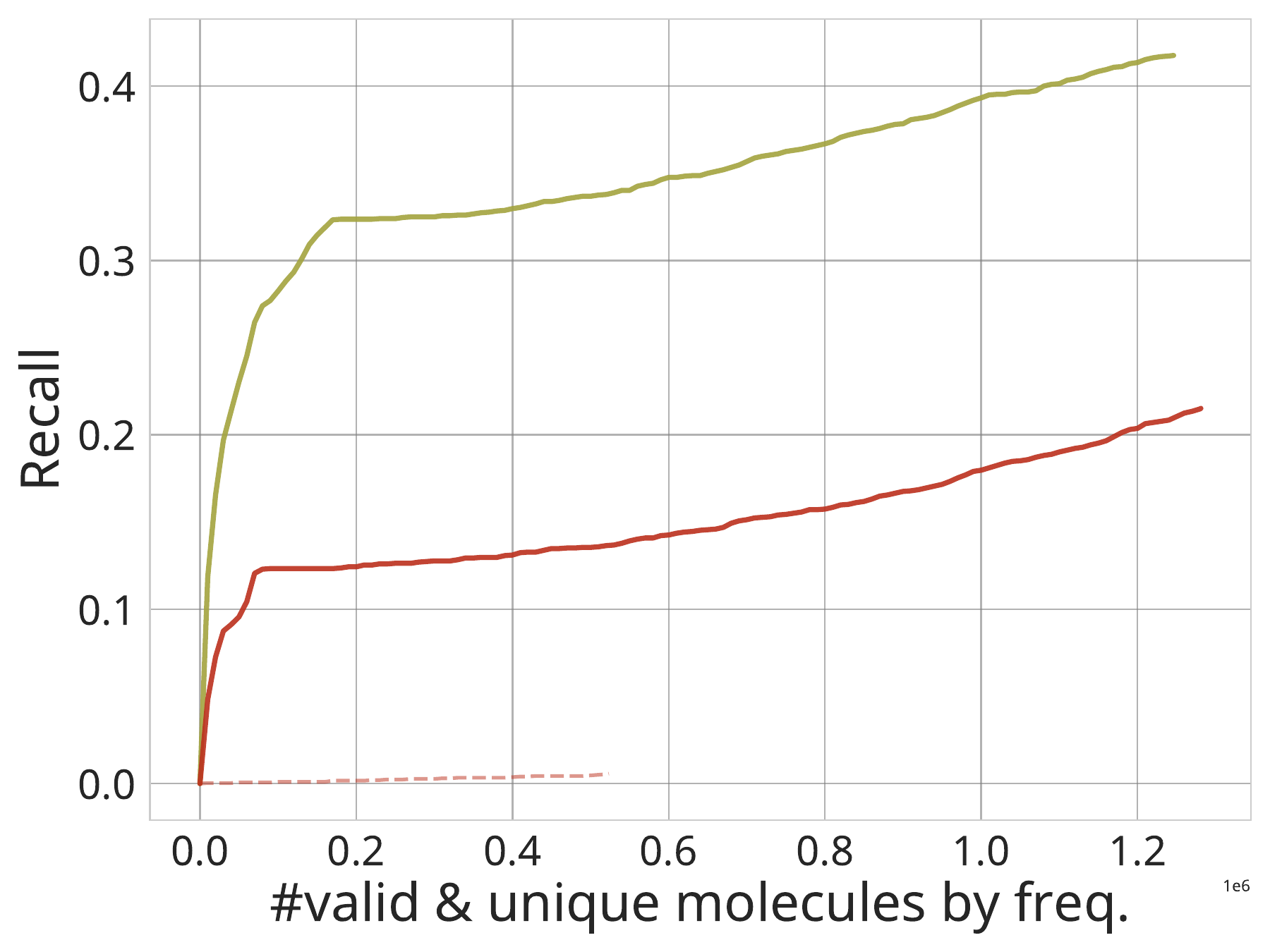}
        \caption{ChEMBL 15}
    \end{subfigure}
    \begin{subfigure}[b]{0.32\textwidth}
        \centering
        \includegraphics[width=\columnwidth]{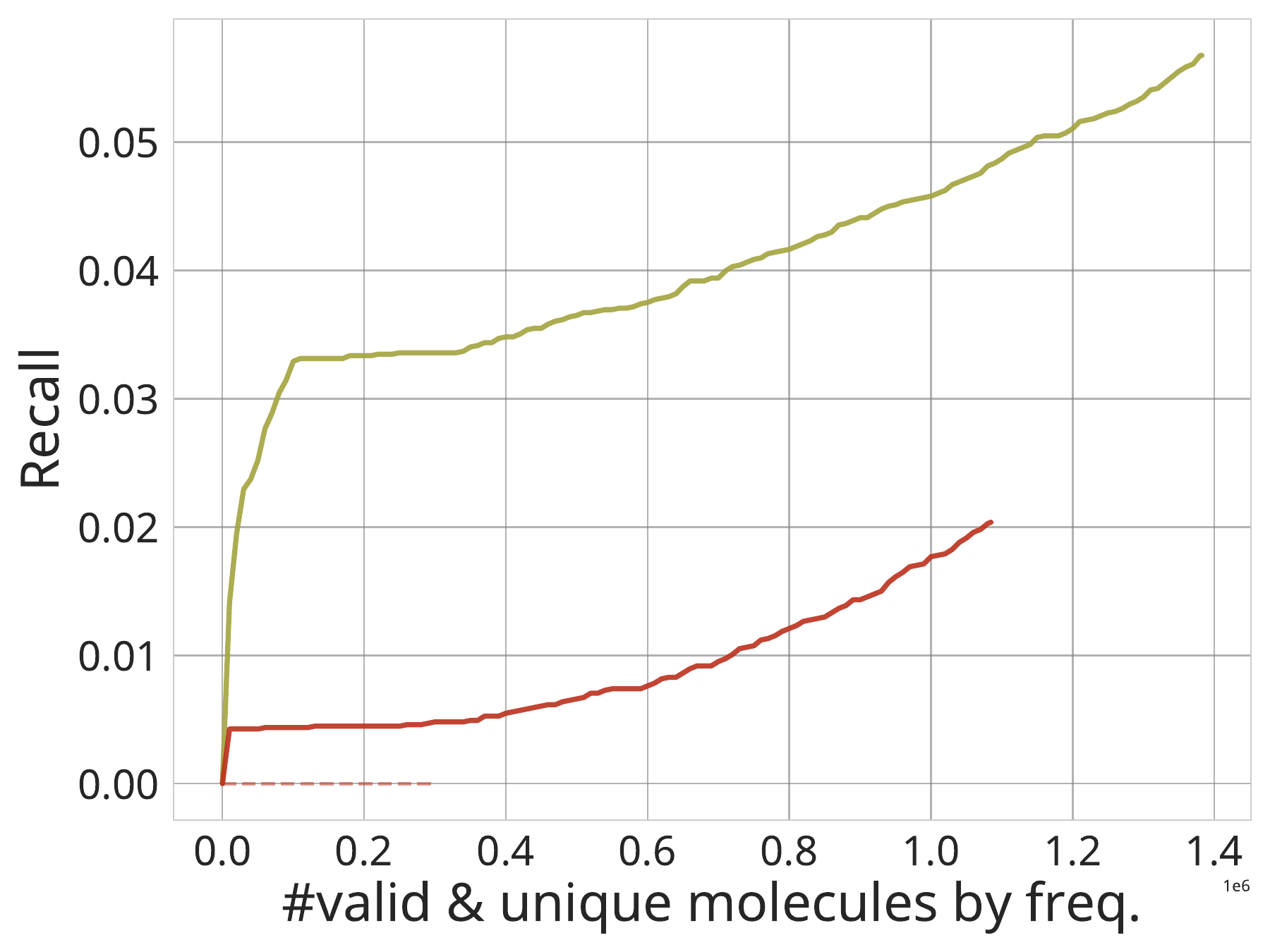}
        \caption{ChEMBL 20}
    \end{subfigure}

    \caption{Recall curves on ChEMBL: NMLN* vs.\ DiGress (discrete and continuous).}
    \label{fig:main-text-metrics_q3-chembl}
\end{figure*}

\paragraph{Baselines.}
We compare against the strongest available NMLN baseline (DeepSet-NMLN \citep{jung2024quantified}) and a general diffusion model for graphs (DiGress; discrete and continuous \citep{digress}). Additional molecule-specific generators (MoleculeRNN variants; PaccMann) are reported in Appendix~\ref{app:exp-details}.

\begin{figure}[t]
    \centering
    \begin{subfigure}{\columnwidth}
        \centering
        \includegraphics[width=0.4\linewidth]{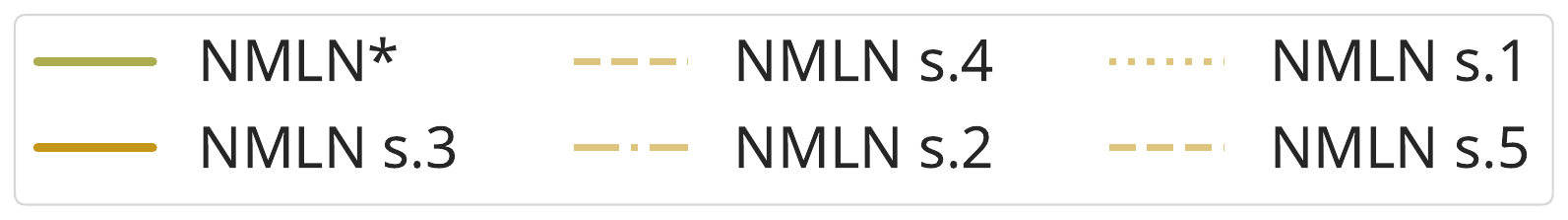}
    \end{subfigure}
    \begin{subfigure}{0.38\textwidth}
        \centering
        \includegraphics[width=\linewidth]{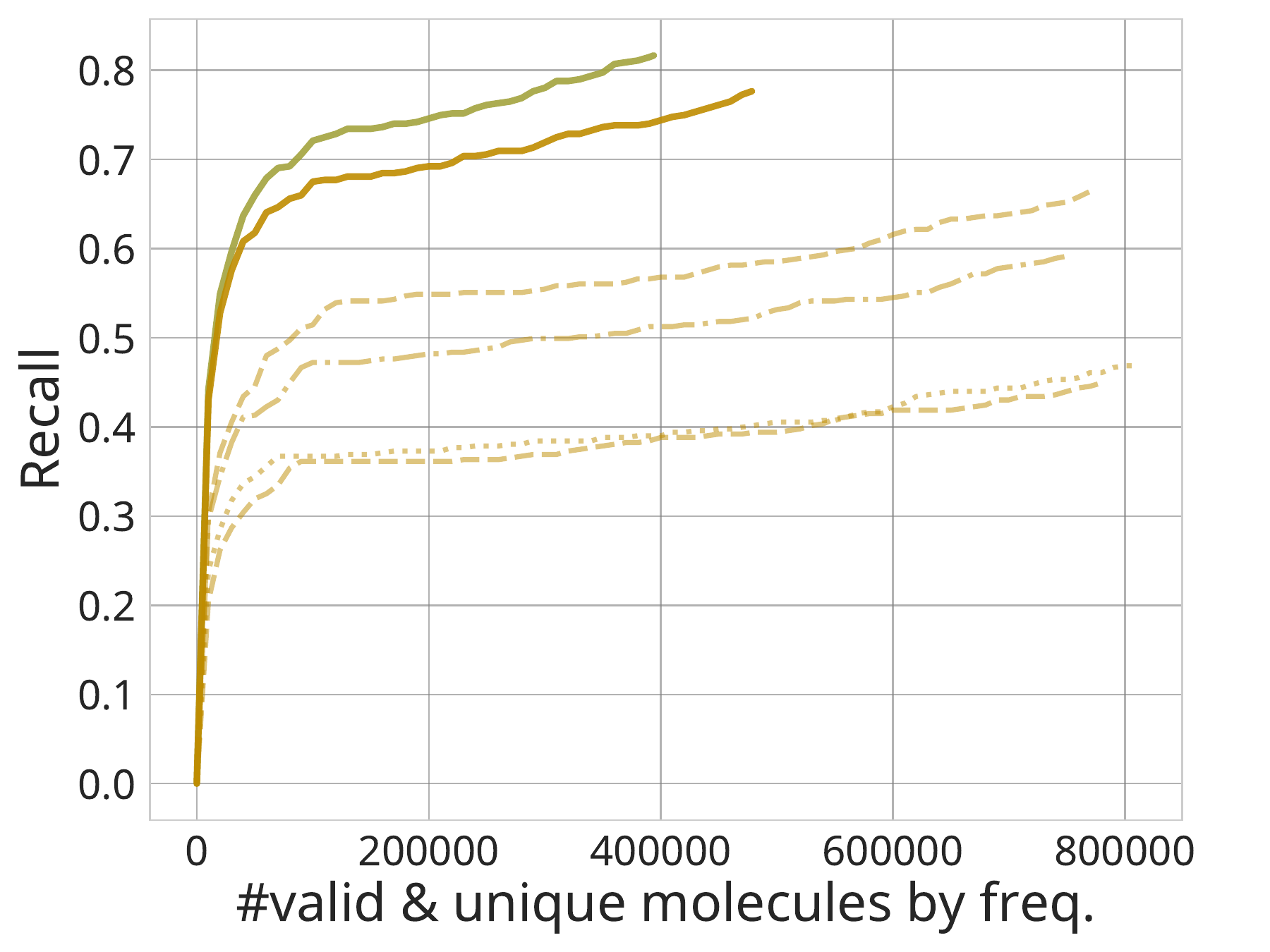}
    \end{subfigure}

    \caption{ChEMBL-10 recall curves. NMLN* (GNN + PN) improves recall over prior NMLN variants; full results and ablations are in the appendix.}
    \label{fig:main-chembl10}
\end{figure}

\subsubsection{Q1: Can NMLN* outperform state-of-the-art diffusion models?}
We compare NMLN* to DiGress \citep{digress} (continuous and discrete). Figure~\ref{fig:main-text-metrics_q3-chembl} shows that NMLN* substantially improves recall across ChEMBL sizes.

\subsubsection{Q2: How much do the improvements matter?}
To disentangle \textbf{stronger potentials} implemented as GNNs from \textbf{better inference} due to parallel noising, we run compute-aware ablations.

\paragraph{Baseline (old) NMLNs.}
We consider \textbf{NMLN s.1}, the DeepSet-NMLN of \citet{jung2024quantified} (8 parallel chains), and \textbf{NMLN s.2}, which increases this to 40 chains (still practical to train). NMLN* uses $N{=}5$ noise levels with $R{=}100$ replicas (500 chains total) and reports samples from the coldest level. Scaling DeepSet-NMLN further (e.g., 100 chains) was prohibitively slow due to dense local potentials.

\paragraph{Effect of the potential family (sampler fixed).}
Keeping plain multi-chain Gibbs sampling, \textbf{NMLN s.3 (GNN + Gibbs)} replaces DeepSet potentials with GNN potentials at the default chain count.

\paragraph{Effect of the sampler (potential fixed).}
Fixing the GNN potential, we vary sampling strategy and compute:
\begin{itemize}
    \item \textbf{NMLN s.4 (GNN + Gibbs, 100 chains)} matches the number of coldest-level replicas in NMLN*.
    \item \textbf{NMLN s.5 (GNN + Gibbs, 500 chains)} matches the total chain count ($N\times R$) of NMLN*.
    \item \textbf{NMLN* (GNN + PN)} runs the full parallel-noising ladder ($N{=}5$, $R{=}100$) with swap moves.
\end{itemize}

\paragraph{Results.}
Figure~\ref{fig:main-chembl10} shows that \textit{both improvements matter}, and that their \textit{combination is clearly the best-performing configuration}.

First, replacing DeepSet potentials with GNN potentials already yields a clear improvement under the same plain Gibbs sampler (compare \textbf{s.3} against the older NMLN baselines \textbf{s.1} and \textbf{s.2}).

Second, improving the sampler with parallel noising (PN) provides additional gains beyond what can be explained by simply running more Gibbs chains. In particular, \textbf{s.4} and \textbf{s.5} increase the number of Gibbs chains to match the coldest-level replicas and the total chain count of NMLN*, respectively, yet \textbf{NMLN* (GNN + PN) still performs better}. Actually, simply increasing the number of parallel chains without a replica exchange scheme seems to hurt performance. Therefore, the improvement is \textbf{not} a pure compute/chain-count effect: \textbf{PN yields better inference than multi-chain Gibbs at matched compute}.

\section{Conclusions}\label{sec:conclusions}

We introduced NMLN*, a strengthened Neural Markov Logic Network that tackles two practical bottlenecks: limited expressivity of fragment-based potentials and brittle, slow-mixing Gibbs inference under near-deterministic learned constraints. On the modeling side, we instantiate NMLN potentials with global GNN energies to capture long-range structure. On the inference side, we propose \emph{parallel noising}, an exact replica-exchange sampler that couples chains across a ladder of fixed corruption levels rather than temperatures, which stays well-behaved even as the target distribution changes during training. Empirically, NMLN* improves recall and sample quality on molecular generation benchmarks.

\FloatBarrier

\bibliographystyle{plainnat}
\bibliography{references}

\appendix
\section*{Appendix}

\section{Parallel Noising versus Parallel Tempering}

\subsection{Parallel Tempering MCMC}\label{sec:parallel_tempering_app}

Parallel tempering MCMC \citep{swendsen1986replica,geyer1991markov,hukushima1996exchange,desjardins2010tempered}, which is also known as {\em replica exchange MCMC}, is a classical method for sampling from complex distributions where plain Gibbs sampling would be inefficient. It improves mixing by simulating several Markov chains in parallel, which may exchange samples (explained below). Each chain is associated with an inverse temperature parameter $\beta_i \in (0,1]$, where the chain with $\beta_k = 1$ targets the true distribution we want to sample from---lower $\beta_i$'s correspond to flatter versions of this distribution.

Formally, each chain samples from the tempered distribution
\[
P_i(x) = \frac{1}{Z_i} \exp(-\beta_i E(x)),
\]
where $E(x)$ is the energy function of the target distribution and $Z_i$ is the normalization constant, known as {\em partition function}. The idea is that chains at lower inverse temperatures (i.e., higher temperatures) are better at exploring the space due to the flattened energy landscape, while the chain at $\beta_k = 1$ provides samples from the true target.

To allow information sharing between chains, the algorithm performs \emph{swap} attempts between adjacent chains $x^{(i)}$ and $x^{(i+1)}$. The proposed swap $(x^{(i)}, x^{(i+1)}) \mapsto (x^{(i+1)}, x^{(i)})$ is accepted with probability
\begin{multline*}
\alpha = \min\left\{1,\; \frac{P_i(x^{(i+1)}) \cdot P_{i+1}(x^{(i)})}{P_i(x^{(i)}) \cdot P_{i+1}(x^{(i+1)})} \right\} \\
= \min\left\{1,\; \exp\Big(-\beta_i E(x^{(i+1)}) + \beta_i E(x^{(i)}) \right. \\
 \left. - \beta_{i+1} E(x^{(i)}) + \beta_{i+1} E(x^{(i+1)}) \Big) \right\}
\end{multline*}
which ensures that each chain samples from the correct corresponding distribution. It turns out that it is also beneficial to perform the swapping in ``odd'' and ``even'' phases, where in the ``odd'' phases, the swaps are attempted for chains with odd number $i$ and their neighbor $i+1$, and analogically for even phases \citep{syed2022non}.

The motivation behind this technique is that chains at higher temperatures are more likely to traverse low-probability regions and can thus escape local modes. The swap mechanism then allows this exploratory information to be propagated back to the lower-temperature chains.

\begin{example}
    Figure \ref{fig:two_peaked_distribution} shows, as an illustration, results of sampling from a given bimodal density with and without parallel tempering---as can be seen, parallel tempering helps the sampler to cross the region of low density separating the two modes.

\begin{figure}[t]
    \centering
    \includegraphics[width=0.76\textwidth]{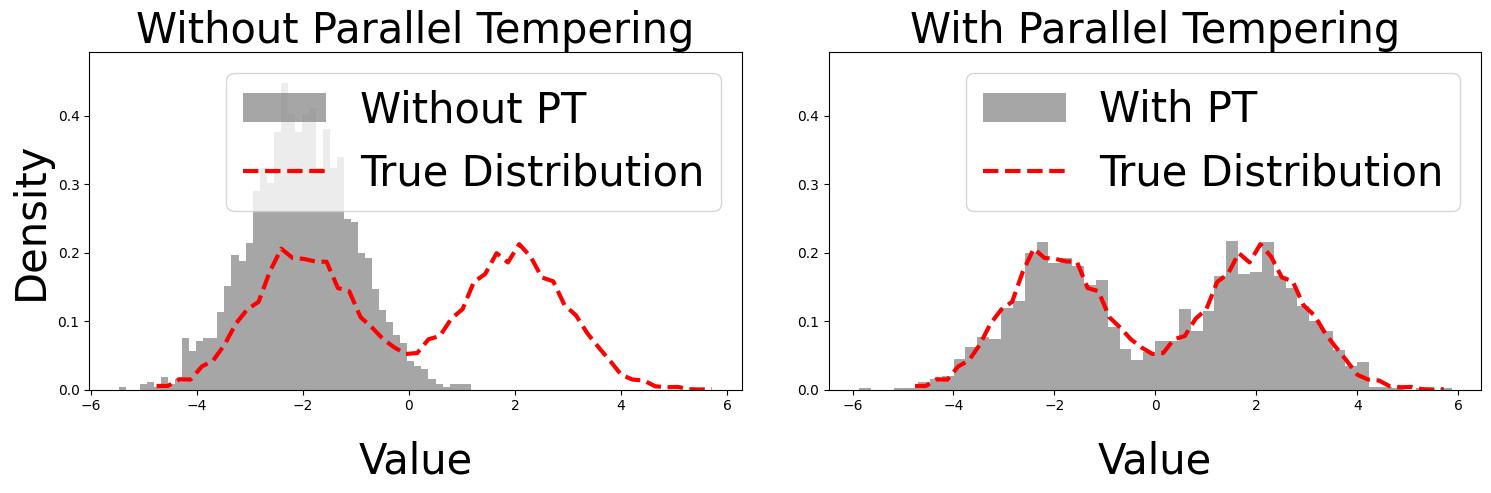}
    \caption{Sampling from a bi-modal distribution: No PT vs. PT}
    \label{fig:two_peaked_distribution}
\end{figure}
\end{example}

\begin{remark}\label{remark1}
It will be convenient for the exposition further in the paper to make the substitution $\rho_i(x) = -\beta_i E(x)$. Here, we just absorbed the inverse temperature parameter and the energy function into the potential function $\rho_i$. The swap probability in parallel tempering then becomes $\alpha = \min\{1,\; \exp\Big(\rho_i(x^{(i+1)}) - \rho_i(x^{(i)}) + \rho_{i+1}(x^{(i)}) - \rho_{i+1}(x^{(i+1)}) \Big)\}$.
\end{remark}

\subsection{Intuition: Parallel Noising vs.\ Parallel Tempering}
\label{app:pn-vs-pt-intuition}

Parallel noising (PN) and parallel tempering (PT) are both replica-exchange methods: they run a ladder of auxiliary chains and use swap moves to transfer exploratory states back toward the target distribution. The crucial difference is how the auxiliary distributions are constructed.

PT builds its ladder by rescaling the energy with different temperatures. High-temperature chains flatten the energy landscape and can therefore move more easily between modes. This works well when the temperature ladder is well matched to the energy scale of the target: adjacent chains must have enough distributional overlap for swap moves to be accepted. However, this requirement can become problematic when the target is very sharp, highly constrained, or changing during training. In such cases, a fixed temperature ladder may either be too coarse, leading to low swap acceptance, or too mild, giving hot chains that are not exploratory enough.

PN instead builds its ladder by corrupting the state space. Rather than changing the energy scale, it defines auxiliary chains corresponding to increasingly noised versions of the structured object. The higher-noise chains explore relaxed versions of the same combinatorial problem, while the lowest-noise chain remains the target of interest. This can be more robust when the difficulty comes from sharp symbolic or combinatorial constraints, because the ladder is tied to the structure of the data rather than to a temperature scale that may be hard to tune.

The next two subsections illustrate this distinction in controlled settings. The first diagnostic uses a synthetic Boolean distribution whose energy scale changes over time, showing how PT can become sensitive to the temperature schedule. The second diagnostic uses a pixel-based $n$-queens generation task with CNN potentials, where the goal is to sample configurations with few or no queen attacks.

\subsection{Time-varying Boolean tracking example}
\label{app:toy-boolean-tracking}
This toy example isolates the issue of changing energy scales during learning.  It shows how a fixed temperature ladder can become poorly matched to a target distributions whose low energy barriers grow over time due to learning from deterministic distributions.

\paragraph{Target.}
Let $x\in\{0,1\}^d$ and let $k=\|x\|_1$ be its Hamming weight.  At time $t$ we define
\begin{align*}
\pi_t(x)&\propto \exp\{-E_t(x)\},\\
E_t(x)&\equiv E_t(k)=\beta_t\min\{k,d-k\}-h_t(2k-d).
\label{eq:toy_energy}
\end{align*}
The first term creates two wells around $0^d$ and $1^d$; increasing $\beta_t$ raises the barrier between them.  The second term tilts the distribution toward one well, and we sweep $h_t$ from negative to positive so that the preferred well changes during the run.  We track the magnetization
\[
m(x)=\frac{1}{d}\sum_{j=1}^d(2x_j-1)=\frac{2k-d}{d}\in[-1,1],
\]
whose exact expectation can be computed by summing over Hamming weights with multiplicities $\binom{d}{k}$.

\paragraph{PT and PN configurations.}
For PT, the replicas target $\pi_{t,T_\ell}(x)\propto\exp\{-E_t(x)/T_\ell\}$ using a geometrically spaced temperature ladder $T_\ell=T_{\max}^{\ell/(L-1)}$, with $T_{\max}=4$ and $L=10$.  For PN, the replicas target corrupted distributions $\pi_{t,\nu}=\pi_tK_\nu$, where $K_\nu$ is the independent bit-flip channel.  We use the quadratic noise schedule $\nu_\ell=\tfrac12(\ell/(L_\nu-1))^2$, which is denser near the clean target.  Both samplers use local bit-flip moves and odd--even adjacent swaps.

\paragraph{Result.}
Figure~\ref{fig:toy-tracking} shows the cold-replica estimate of $\mathbb{E}_{\pi_t}[m]$, averaged across inner-chain samples and random seeds.  As the barrier increases, PT with a fixed temperature range increasingly lags behind the changing target.  PN tracks the change more closely because high-noise replicas remain close to uniform even when the underlying energy becomes sharp.  This supports the use of corruption levels as a robust auxiliary ladder when training gradually learns near-deterministic constraints.

\begin{figure}[t]
  \centering
  \includegraphics[width=0.76\linewidth]{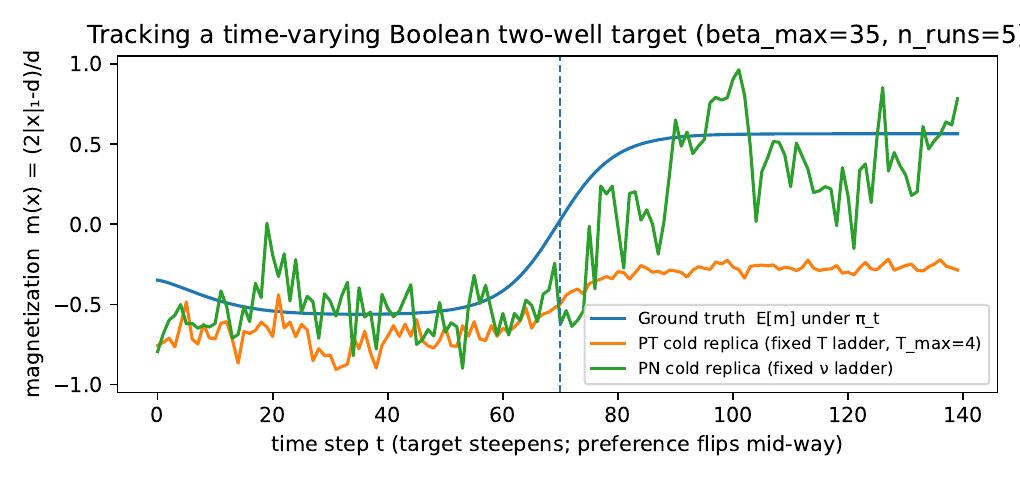}
  \caption{\textbf{Tracking a time-varying Boolean two-well target.}
  Ground truth $\mathbb{E}_{\pi_t}[m]$ is computed exactly by summing over Hamming weights.  PT uses a fixed temperature ladder, while PN uses a fixed noising ladder.  Curves for PT and PN average the cold-replica magnetization over inner-chain samples and multiple random seeds.}
  \label{fig:toy-tracking}
\end{figure}

\subsection{Pixel-based n-queens generation}
\label{app:nqueens}
We also evaluate PN and PT on a pixel-based n-queens generation task.  A board is represented as a binary image, with a queen indicated by an active pixel.  The energy is parameterized by a convolutional neural network, so the task tests whether the sampler can explore a learned, image-level potential whose high-probability states satisfy a global combinatorial constraint.

\paragraph{Metric.}
For each generated board, we count the number of attacking queen pairs, i.e., pairs of active pixels that share a row, column, or diagonal.  We report the average number of attacks over generated samples; lower is better, and zero attacks corresponds to a solved board when the board also contains the required number of queens.  This metric directly measures how much hard constraint violation remains in the generated samples.

\paragraph{Comparison.}
We use the same CNN potential and the same within-level local updates for both methods.  PT uses a temperature ladder beginning at the clean target ($T=1$) and increasing through\\ $\{1.0,1.08,1.1664,1.259712,1.360489\}$.  PN uses the corruption ladder $\nu\in\{0,0.02,0.05,0.1,0.2\}$.  In both cases, we report the clean/cold chain, and the curves correspond to independent runs.

\paragraph{Result.}
Figure~\ref{fig:nqueens-attacks} shows that both methods rapidly reduce the number of attacks from the random initialization, but their long-run behavior differs.  PT plateaus at a higher number of remaining conflicts, whereas PN continues to reduce the average number of attacks across the run.  The final samples in Figure~\ref{fig:nqueens-samples} give a qualitative view of the same effect: high-noise PN levels remain exploratory, while the clean PN chain produces less conflicted boards than the corresponding cold PT chain.  This diagnostic supports the main claim that noising ladders can be better suited than fixed temperature ladders for learned energies with crisp combinatorial structure.

\begin{figure}[t]
    \centering
    \includegraphics[width=0.76\linewidth]{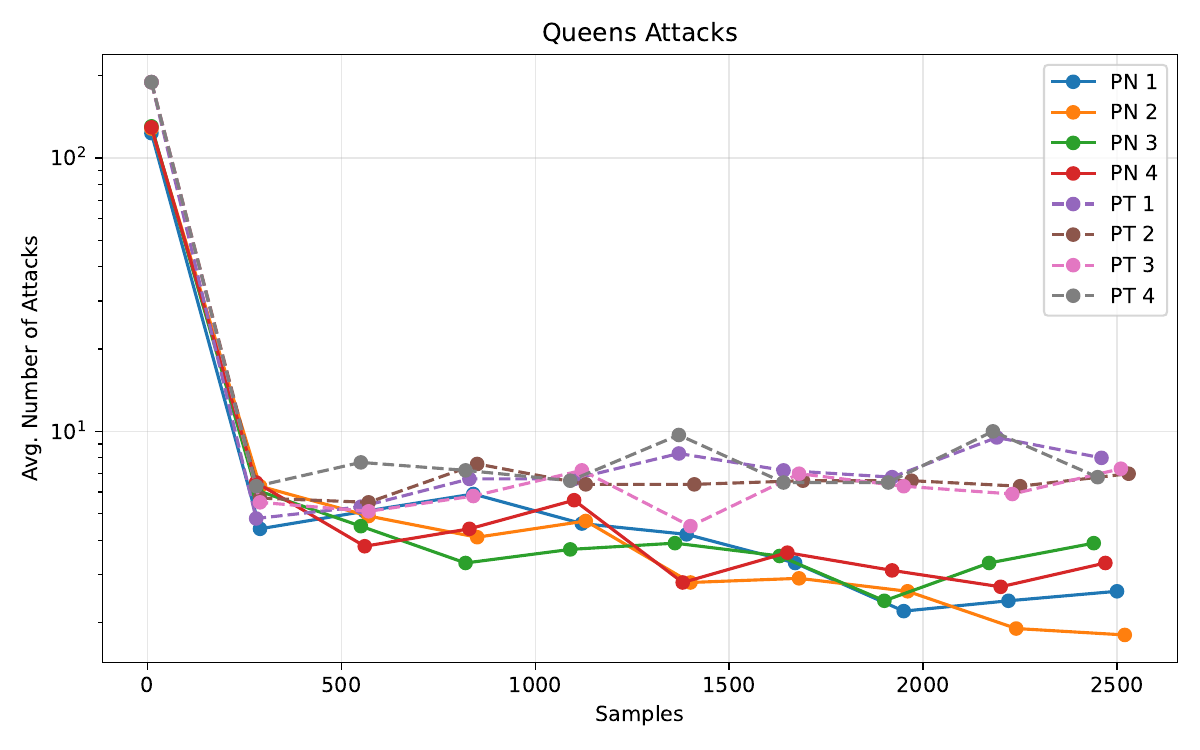}
    \caption{\textbf{Pixel-based n-queens diagnostic.} Average number of attacking queen pairs as a function of generated samples.  Lower is better.  PN reaches consistently lower attack counts than PT after the initial burn-in, indicating better exploration of low-conflict boards under the learned CNN potential.}
    \label{fig:nqueens-attacks}
\end{figure}

\begin{figure*}[t]
    \centering
    \begin{subfigure}[b]{0.48\textwidth}
        \centering
        \includegraphics[width=\linewidth]{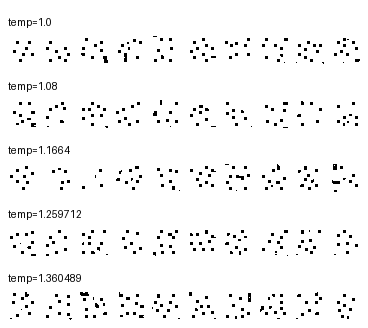}
        \caption{PT levels.}
    \end{subfigure}
    \hfill
    \begin{subfigure}[b]{0.48\textwidth}
        \centering
        \includegraphics[width=\linewidth]{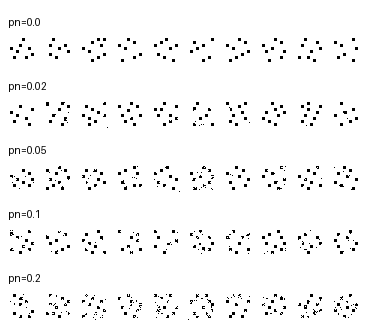}
        \caption{PN levels.}
    \end{subfigure}
    \caption{Representative final n-queens samples at different temperature/noise levels.  The cold/clean level is shown at the top of each panel.}
    \label{fig:nqueens-samples}
\end{figure*}

\section{Additional experimental results and ablations}\label{app:exp-additional}\label{app:exp-details}

\subsection{Compared Models}

Besides the latest version of NMLN, we compare NMLN* with the following two diffusion-based algorithms from \citep{digress}:

\begin{itemize}
    \item \textbf{DiGress Discrete} \citep{digress}, implemented in \citep{vignac2023digressgithub}, utilizes a discrete diffusion process that progressively edits graphs with noise, through the process of adding or removing edges and changing the labels.
    \item \textbf{DiGress Continuous} which is the continuous version using Gaussian noise.
\end{itemize}

We also compare with molecule-specific generators that use string encodings such as SMILES or SELFIES \citep{selfies}. SELFIES guarantees molecular validity for strings over its alphabet, whereas SMILES does not. These methods exploit chemistry-specific representations and therefore do not apply to arbitrary relational domains, but they provide informative domain-specialized baselines:

\begin{itemize}
    \item \textbf{Molecule-RNN Char} \citep{shi2025moleculernn} is a recurrent neural network designed to generate novel molecules from the distribution of a molecular training dataset. This method is based on the SMILES representation.
    \item \textbf{Molecule-RNN Regex} A more structured approach using regular expressions, where elements such as square-bracketed atoms (e.g., "[O-]") are treated as individual tokens.
    \item \textbf{Molecule-RNN Selfies} \citep{selfies} uses the Selfies representation of molecules---this model therefore generates only valid molecules.
    \item \textbf{PaccMann} \citep{born2021paccmannrl}, implemented by the PaccMannRL pipeline \citep{paccmann_rlgithub}, is a VAE-based generative model guided by reinforcement learning, where a learned reward function predicts molecule--target interactions, enabling the targeted generation of novel compounds for given protein or transcriptomic inputs.
\end{itemize}

\subsection{Datasets}

We report experimental results on four datasets restricted to molecules of sizes 9, 10, 15, and 20. These datasets are:

\begin{itemize}
    \item \textbf{ChEMBL} \citep{chembl}, which is a manually curated database of bioactive molecules with drug-like properties,

    \item \textbf{QM9} \citep{doi:10.1021/ci300415d,ramakrishnan2014quantum} contains stable small organic molecules made up of carbon, hydrogen, oxygen, nitrogen, and fluorine atoms,

    \item \textbf{ZINC250k} \citep{zinc250k}, which is a free database of commercially available compounds for virtual screening, and

    \item \textbf{MOSES} \citep{moses}, which is a curated and cleaned subset of the ZINC database, prepared specifically for benchmarking molecular generative models.
\end{itemize}

The number of molecules in the subsets of these datasets which we use are reported in Table \ref{tab:mol_sizes_app}.

\begin{table}[th]
\centering
\begin{tabular}{lcccc}
\hline
\textbf{Dataset} & \textbf{Size 9} & \textbf{Size 10} & \textbf{Size 15} & \textbf{Size 20} \\
\hline
ChEMBL     & 1 684   & 2 589 & 14 658 & 44 222 \\
QM9        & 109 813 & -     & -      & -      \\
MOSES      & -       & -     & 579    & 195 650    \\
ZINC250k   & 28      & 90    & 1 766  & 8 590    \\
\hline
\end{tabular}
\caption{Number of molecules in the datasets with a given number of heavy atoms.}
\label{tab:mol_sizes_app}
\end{table}

\subsection{Methodology}

We mainly follow the experimental methodology from \citep{jung2024quantified} where it was advocated to use a variant of ROC curves, called {\em coverage} or {\em recall curves} (we use the latter term in this paper), and motivated by \citet{sajjadi2018assessing}, obtained as follows: We let the model generate $N$ samples and we collect all the unique ones. For each unique sample, we store its frequency, i.e., how many times they were generated. Here, frequency is understood as a proxy for the probability of the molecule given by the learned model. As a generative performance indicator, we compute how many of the $t$ most frequently generated molecules are in the test set. We then plot this measure w.r.t.\ frequency thresholds $t$. This metric is related to a ROC curve, but in the generative setting.

To ensure a fair comparison, we let all methods generate a fixed number of 2M samples per model and dataset. While some methods may require longer generation times, this can be considered negligible in the context of real-world molecular synthesis where the top-ranking molecules suggested by the models may need to eventually be synthesized and tested.

\subsection{Detailed Experimental Settings}

Datasets are partitioned into training and test sets using an 80/20 split. All models were trained using their default hyperparameters as provided in the respective open-source repositories. The only modifications involved adjusting the generation pipeline to produce 2,000,000 samples, saving outputs in the SMILES format, and, in some cases, disabling third-party post-processing filters that would otherwise exclude invalid samples.

The GNN potentials in NMLN* consist of 10 R-GCN layers with a hidden dimensionality of 128 and ReLU activations. The update function is implemented as a standard linear layer with a matching hidden size and activation. The output of the GNN is formed by aggregating all intermediate layer outputs and applying a final MLP with three linear layers, each using ReLU activations and the same hidden size.

\subsection{Experimental Questions and Results}

In this section, we state several research questions and address them experimentally.

\subsubsection{Q1: Does NMLN* outperform the latest NMLN?}

\begin{figure}
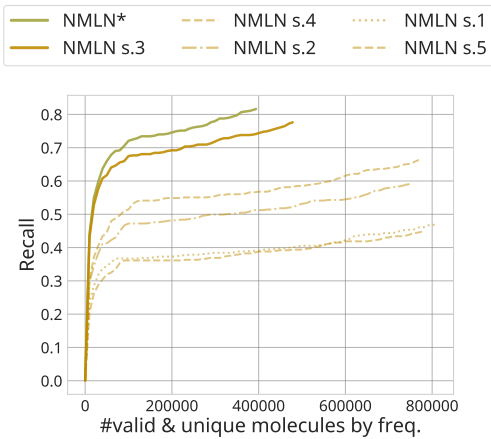

    \centering

    \begin{subfigure}{\columnwidth}
        \centering
        \includegraphics[width=0.4\linewidth]{metrics_q1a2_v2/chembl10_legend_only_two_lines.pdf}
    \end{subfigure}

    \vspace{0.2cm}

    \begin{subfigure}{0.38\columnwidth}
        \centering
        \includegraphics[width=\linewidth]{metrics_q1a2_v2/chembl10_no_legend.pdf}
    \end{subfigure}

    \caption{Recall curves for different variants of NMLN on the dataset ChEMBL 10.}
    \label{fig:metrics_q1a2-chembl}
\end{figure}

We compared NMLN* with two baselines derived from the previous model: NMLN s.1, which uses the DeepSet-NMLN settings, and NMLN s.2, which increases the number of parallel chains from 8 in the DeepSet-NMLN to 40, which is still trainable in a reasonable time. In NMLN*, we run $R{=}100$ replicas over $N{=}5$ noise levels (500 chains total); we report samples from the coldest level (setting it to 100 for the DeepSet-NMLN already lead to unacceptably high runtime, which is a consequence of the dense local potential functions used in the original NMLNs).

As shown in Figure \ref{fig:metrics_q1a2-chembl}, the NMLN* models outperform both NMLN competitors already for molecules with 10 atoms by a significant margin, demonstrating the effectiveness of the proposed modifications.

\subsubsection{Q2: How much do the improvements matter?}

To separate the two contributions---\textbf{stronger potentials} and \textbf{better inference}---we consider a set of compute-aware ablations, all summarized in Figure~\ref{fig:metrics_q1a2-chembl}.

\paragraph{Effect of the potential family (sampler fixed).}
We compare the original DeepSet potentials (NMLN s.1/s.2 from Q1) against GNN potentials while keeping the sampler as plain multi-chain Gibbs:
\begin{itemize}
    \item \textbf{NMLN s.3 (GNN + Gibbs)} replaces the DeepSet potential with the GNN potential but keeps the default chain count.
\end{itemize}

\paragraph{Effect of the sampler (potential fixed).}
We then keep the GNN potential fixed and vary the sampling strategy and compute budget:
\begin{itemize}
    \item \textbf{NMLN s.4 (GNN + Gibbs, 100 chains)} increases the number of Gibbs chains to match the number of replicas at the coldest level of NMLN*.
    \item \textbf{NMLN s.5 (GNN + Gibbs, 500 chains)} matches the \emph{total} number of maintained chains ($N\times R$) in NMLN*.
    \item \textbf{NMLN* (GNN + PN)} uses the full parallel-noising ladder with $N{=}5$ levels and $R{=}100$ replicas (500 chains total), including swap moves.
\end{itemize}

\noindent\textbf{Compute note.} Replica exchange incurs overhead (swap evaluations) beyond within-level updates. To avoid overstating improvements, we report results both at matched chain count (s.4) and at matched total chain count (s.5), and we use the same generation budget (2M produced samples) across all methods.

To keep within-level updates tractable with a global GNN energy, we batch many candidate local edits on the GPU and reuse intermediate message-passing activations whenever possible; for the molecule sizes considered here (up to 20 heavy atoms), a full forward recomputation per proposed edit is also feasible and remains a small fraction of overall training time.

As shown in Figure \ref{fig:metrics_q1a2-chembl}, each individual modification contributes to performance improvements, but the full NMLN* configuration achieves the highest performance overall.

\subsubsection{Q3: Can NMLN* outperform state-of-the-art diffusion models?}

To answer this question, we compared NMLN* with DiGress \citep{digress}, both in the continuous and in the discrete versions. DiGress is another general graph-generation method and therefore a direct competitor to NMLNs. In Appendix Figure~\ref{fig:metrics_q3-chembl}, ~\ref{fig:metrics_q3-MOSES}, ~\ref{fig:metrics_q3-QM}, and ~\ref{fig:metrics_q3-ZINCk}, we compare the recall curves. On most of the datasets, NMLN* significantly outperforms diffusion models. On size 20, the recall is small for both models as the number of possible molecules is huge, but NMLN* still achieves a higher recall. We also observed that NMLN* was able to generate a significantly larger number of valid samples on the larger molecules.

\begin{figure*}[t]
    \centering

    \begin{subfigure}[b]{0.9\textwidth}  
        \centering
        \includegraphics[width=0.4\textwidth]{metrics_q3_v3/chembl20_legend_only.pdf}
    \end{subfigure}
    \vspace{0.2cm}
    \begin{subfigure}[b]{0.4\textwidth}
        \centering
        \includegraphics[width=\columnwidth]{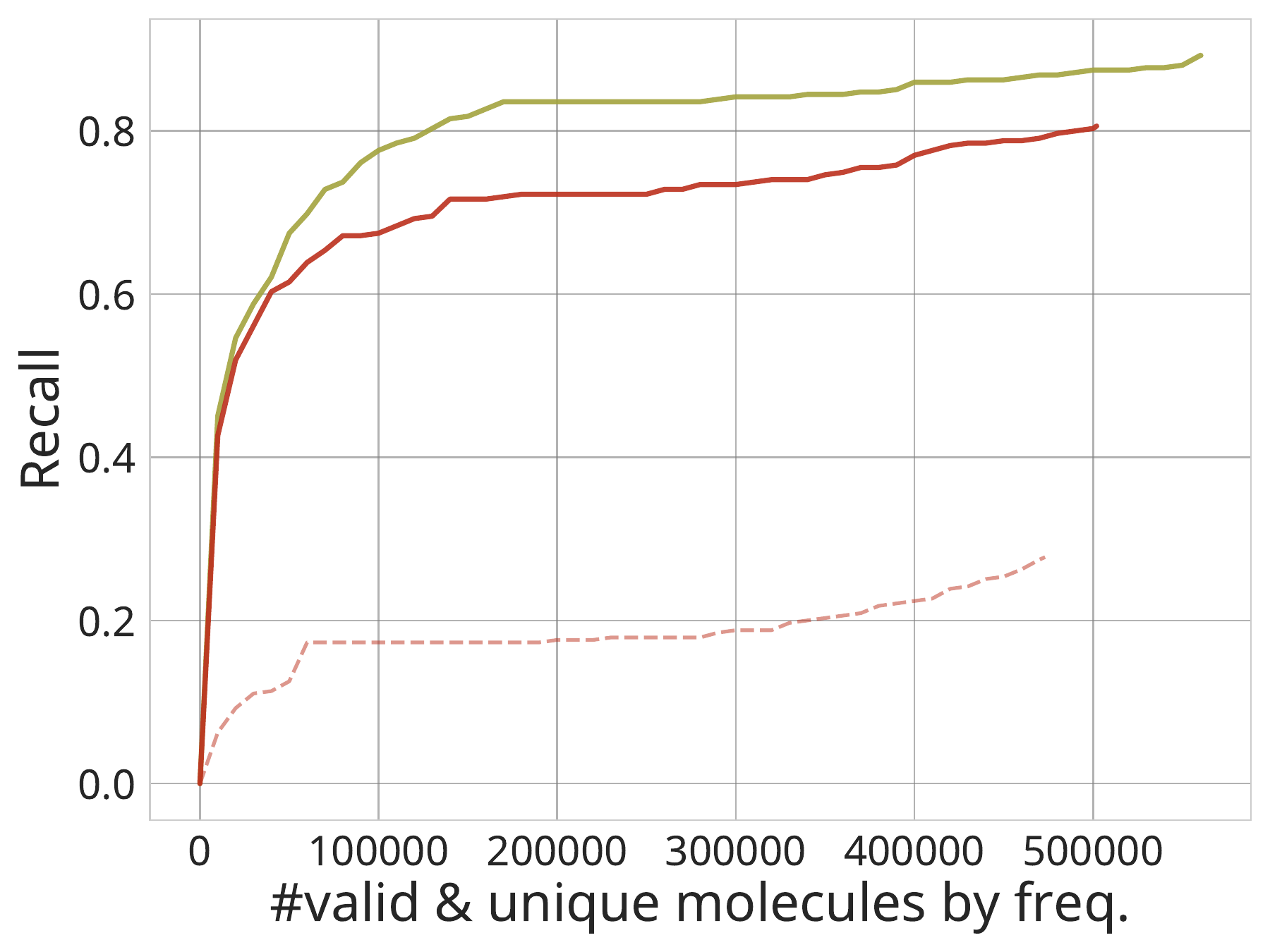}
        \caption{ChEMBL 9}
    \end{subfigure}
    \hfill
    \begin{subfigure}[b]{0.4\textwidth}
        \centering
        \includegraphics[width=\columnwidth]{metrics_q3_v3/chembl10_no_legend.pdf}
        \caption{ChEMBL 10}
    \end{subfigure}
    \hfill
    \begin{subfigure}[b]{0.4\textwidth}
        \centering
        \includegraphics[width=\columnwidth]{metrics_q3_v3/chembl15_no_legend.pdf}
        \caption{ChEMBL 15}
    \end{subfigure}
    \hfill
    \begin{subfigure}[b]{0.4\textwidth}
        \centering
        \includegraphics[width=\columnwidth]{metrics_q3_v3/chembl20_no_legend.pdf}
        \caption{ChEMBL 20}
    \end{subfigure}

    \caption{Recall curves for different sizes on the dataset ChEMBL.}
    \label{fig:metrics_q3-chembl}
\end{figure*}

\begin{figure*}[t]
    \centering
    \begin{subfigure}[b]{0.9\textwidth}  
        \centering
        \includegraphics[width=0.4\textwidth]{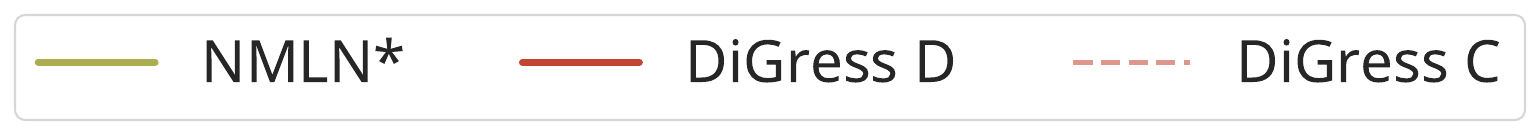}
    \end{subfigure}
    \vspace{0.2cm}
    \begin{subfigure}[b]{0.4\textwidth}
        \centering
        \includegraphics[width=\columnwidth]{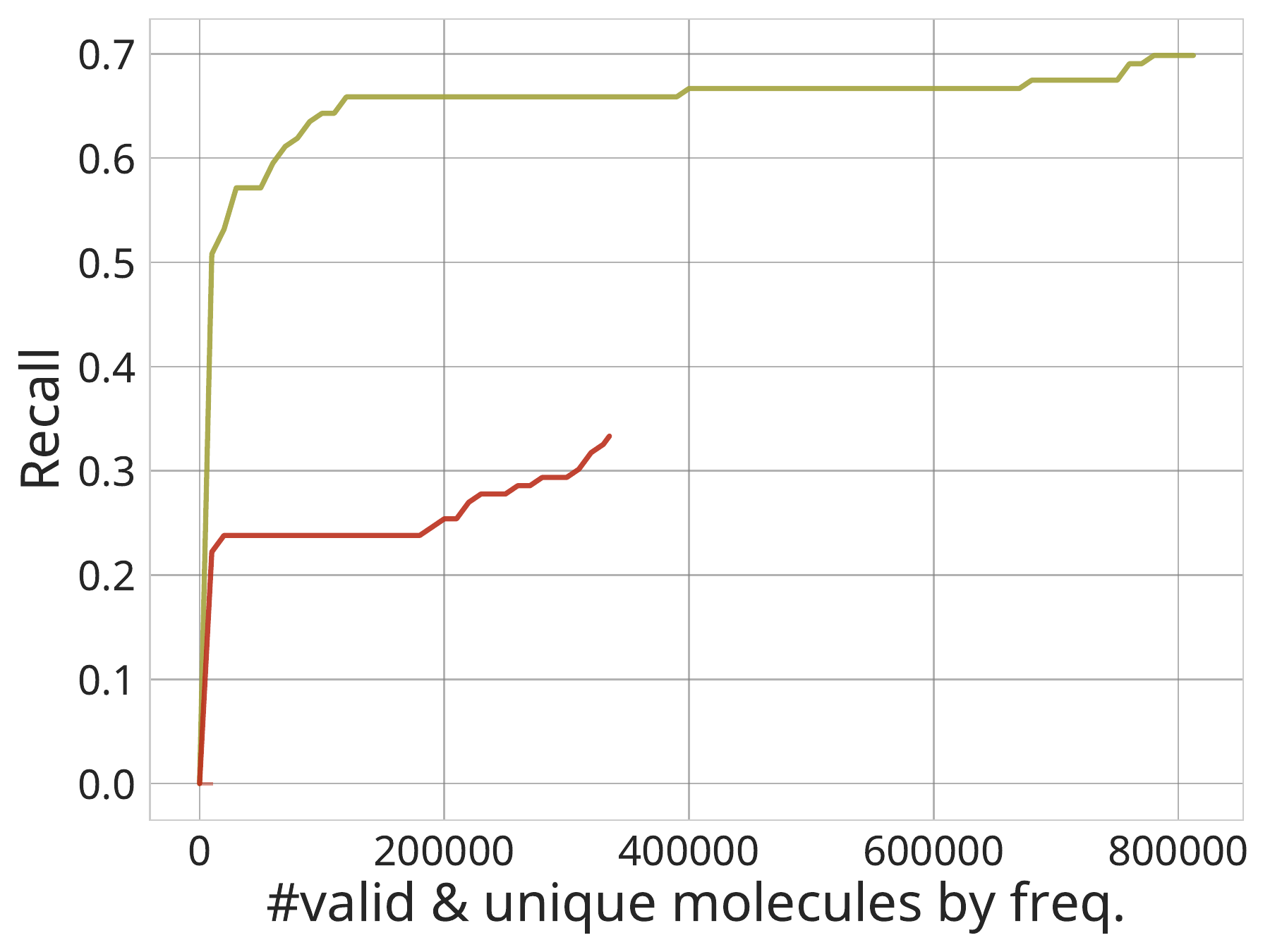}
        \caption{MOSES15}
    \end{subfigure}
    \begin{subfigure}[b]{0.4\textwidth}
        \centering
        \includegraphics[width=\columnwidth]{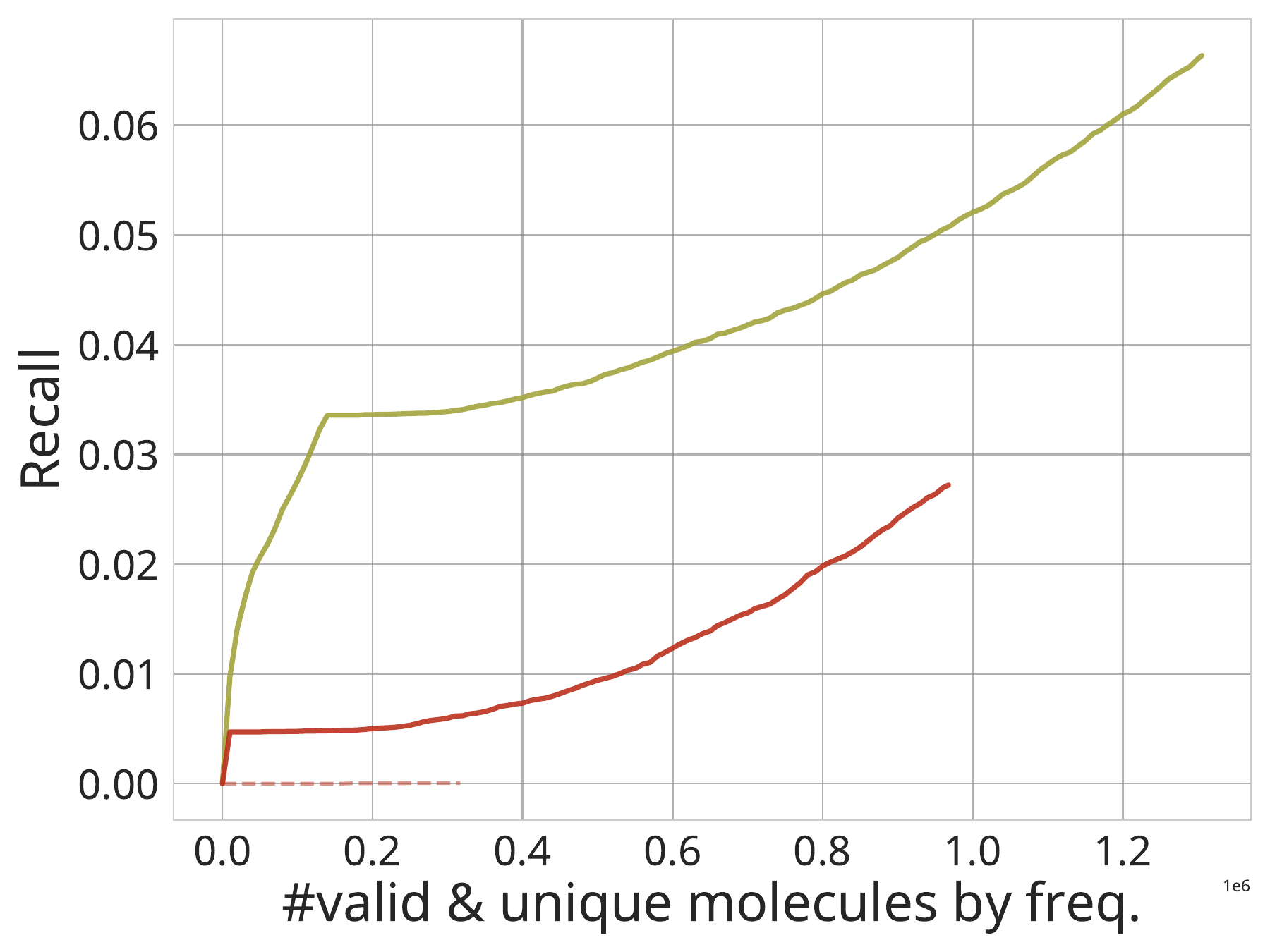}
        \caption{MOSES20}
    \end{subfigure}
    \caption{Recall curves for different sizes on the dataset MOSES.}
    \label{fig:metrics_q3-MOSES}
\end{figure*}

\begin{figure*}[ht]
    \centering
    \includegraphics[width=0.4\textwidth]{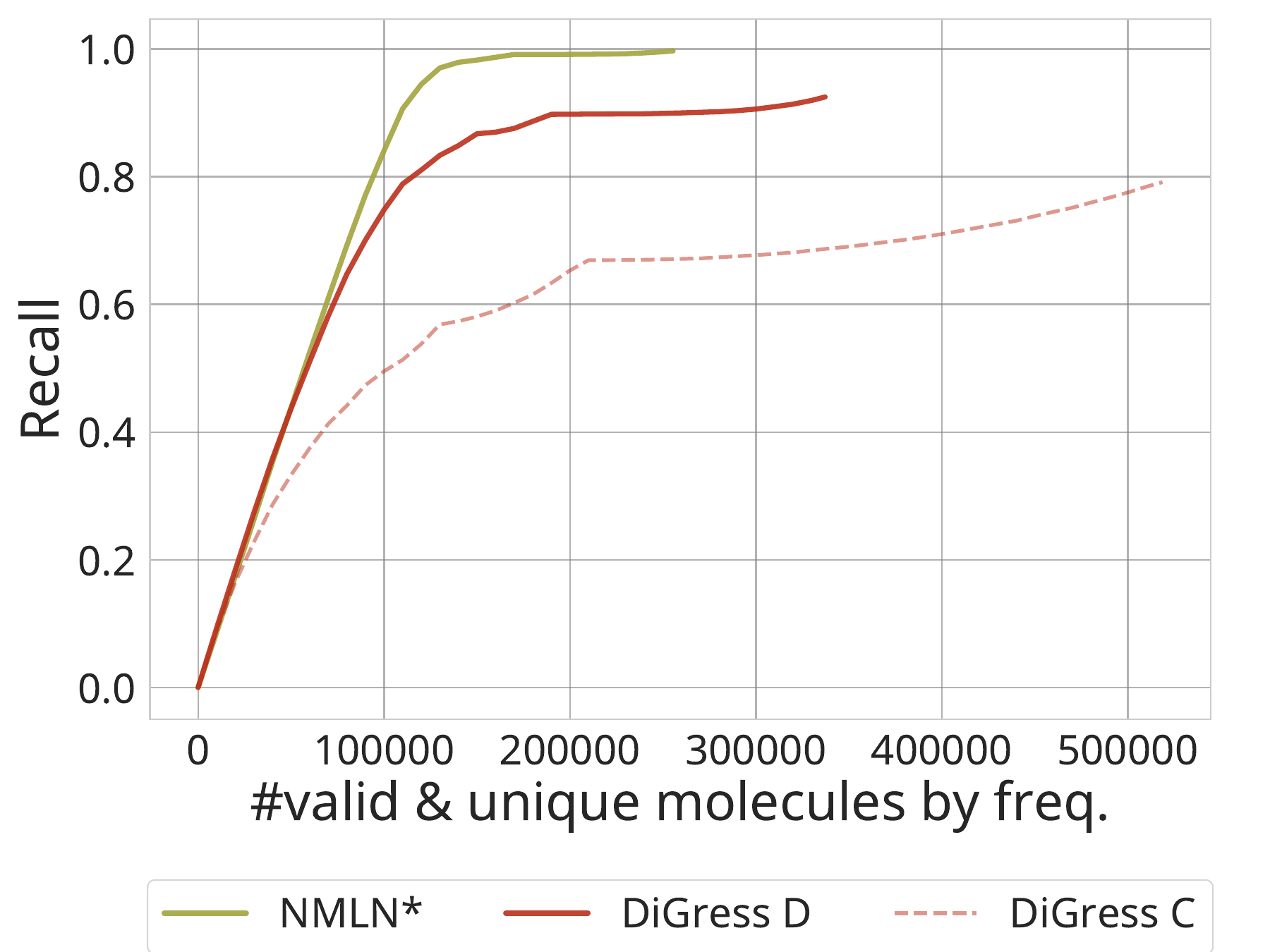}
    \caption{Recall curves for the QM9.}
    \label{fig:metrics_q3-QM}
\end{figure*}

\begin{figure*}[ht]
    \centering
    \begin{subfigure}[b]{0.9\textwidth}  
        \centering
        \includegraphics[width=0.4\textwidth]{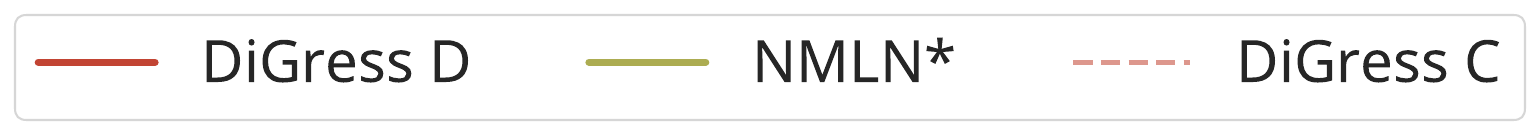}
    \end{subfigure}
    \vspace{0.2cm}
    \begin{subfigure}[b]{0.4\textwidth}
        \centering
        \includegraphics[width=\columnwidth]{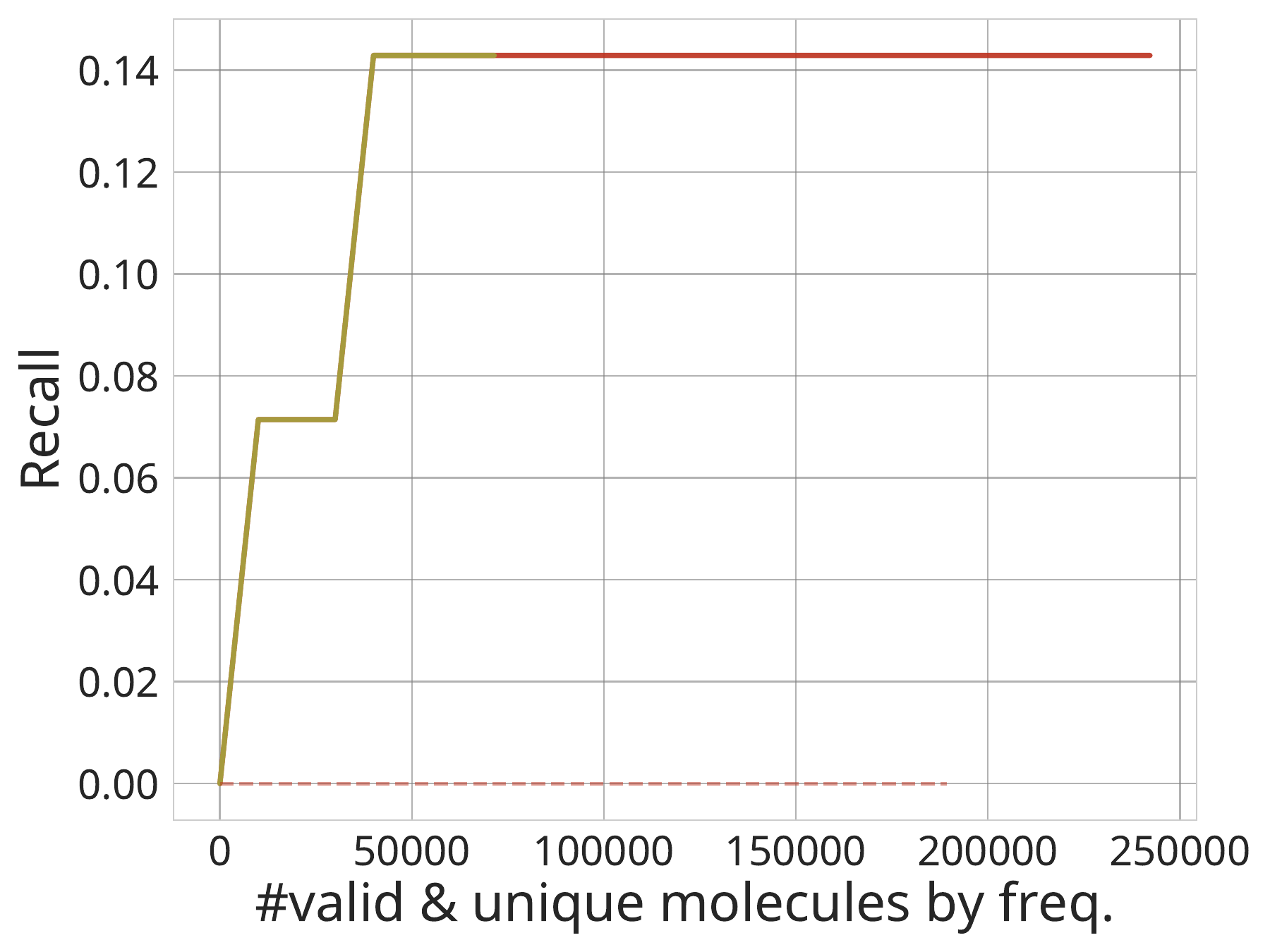}
        \caption{ZINC250k 9}
    \end{subfigure}
    \begin{subfigure}[b]{0.4\textwidth}
        \centering
        \includegraphics[width=\columnwidth]{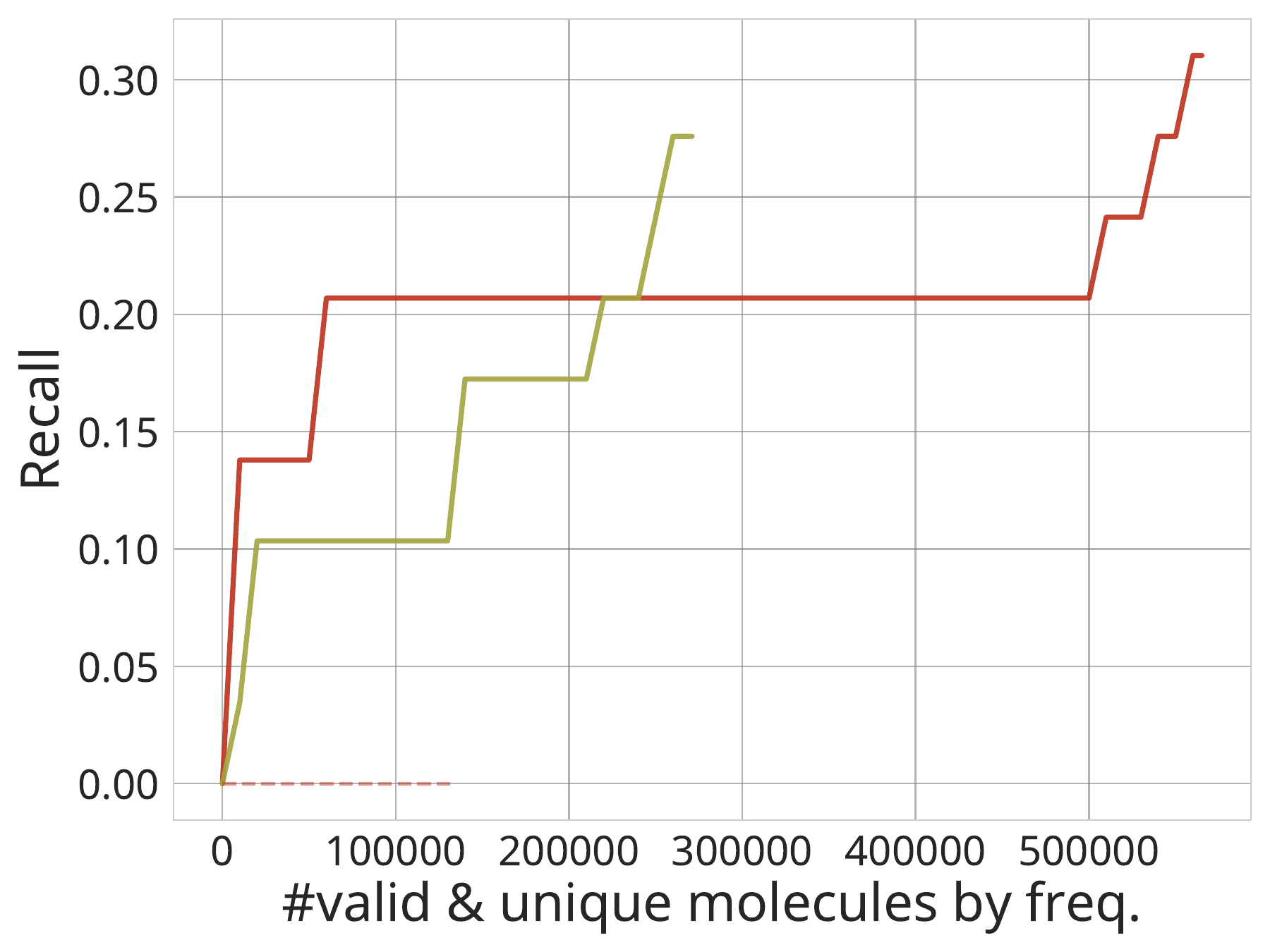}
        \caption{ZINC250k 10}
    \end{subfigure}
    \\
    \begin{subfigure}[b]{0.4\textwidth}
        \centering
        \includegraphics[width=\columnwidth]{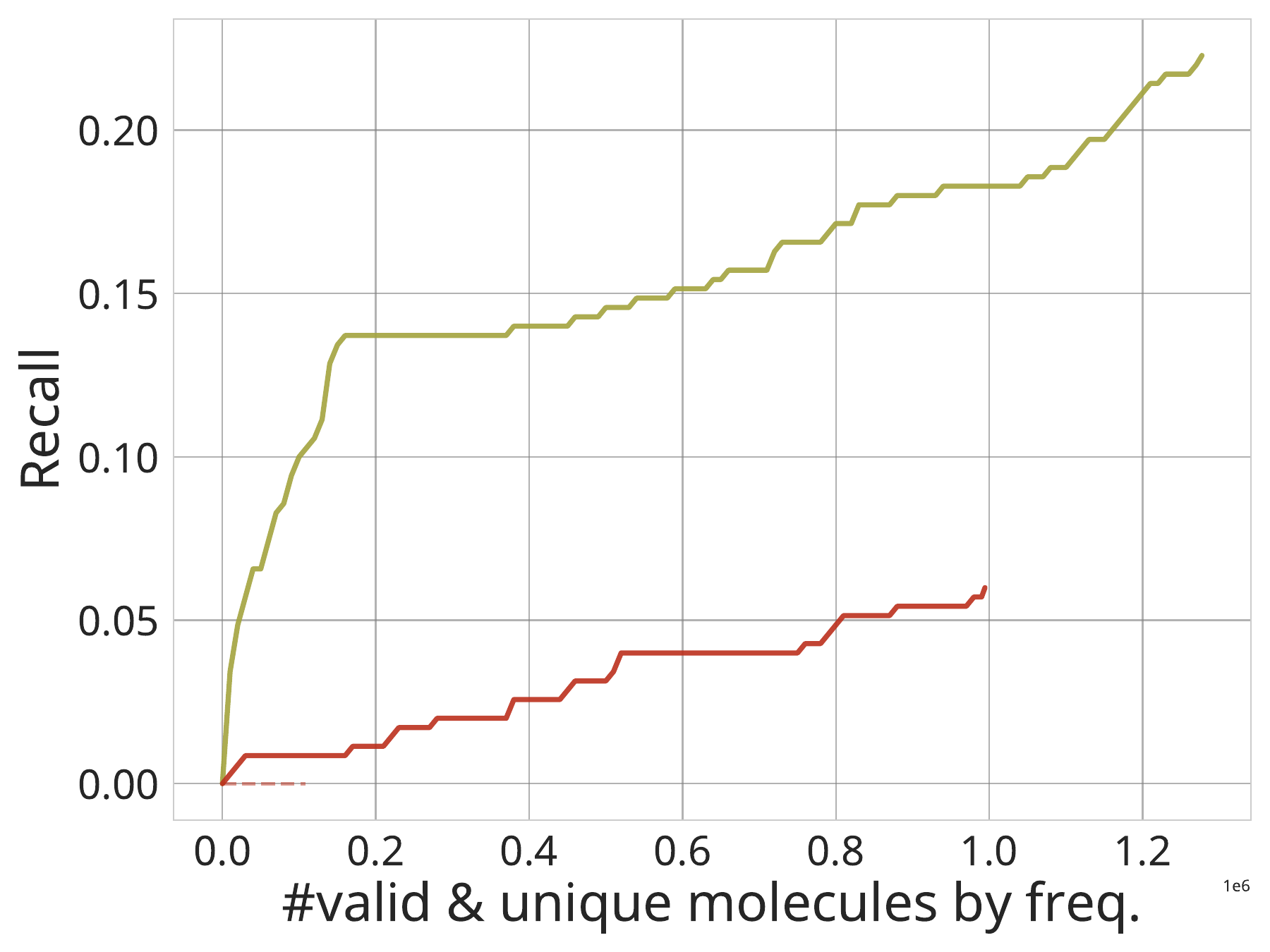}
        \caption{ZINC250k 15}
    \end{subfigure}
    \begin{subfigure}[b]{0.4\textwidth}
        \centering
        \includegraphics[width=\columnwidth]{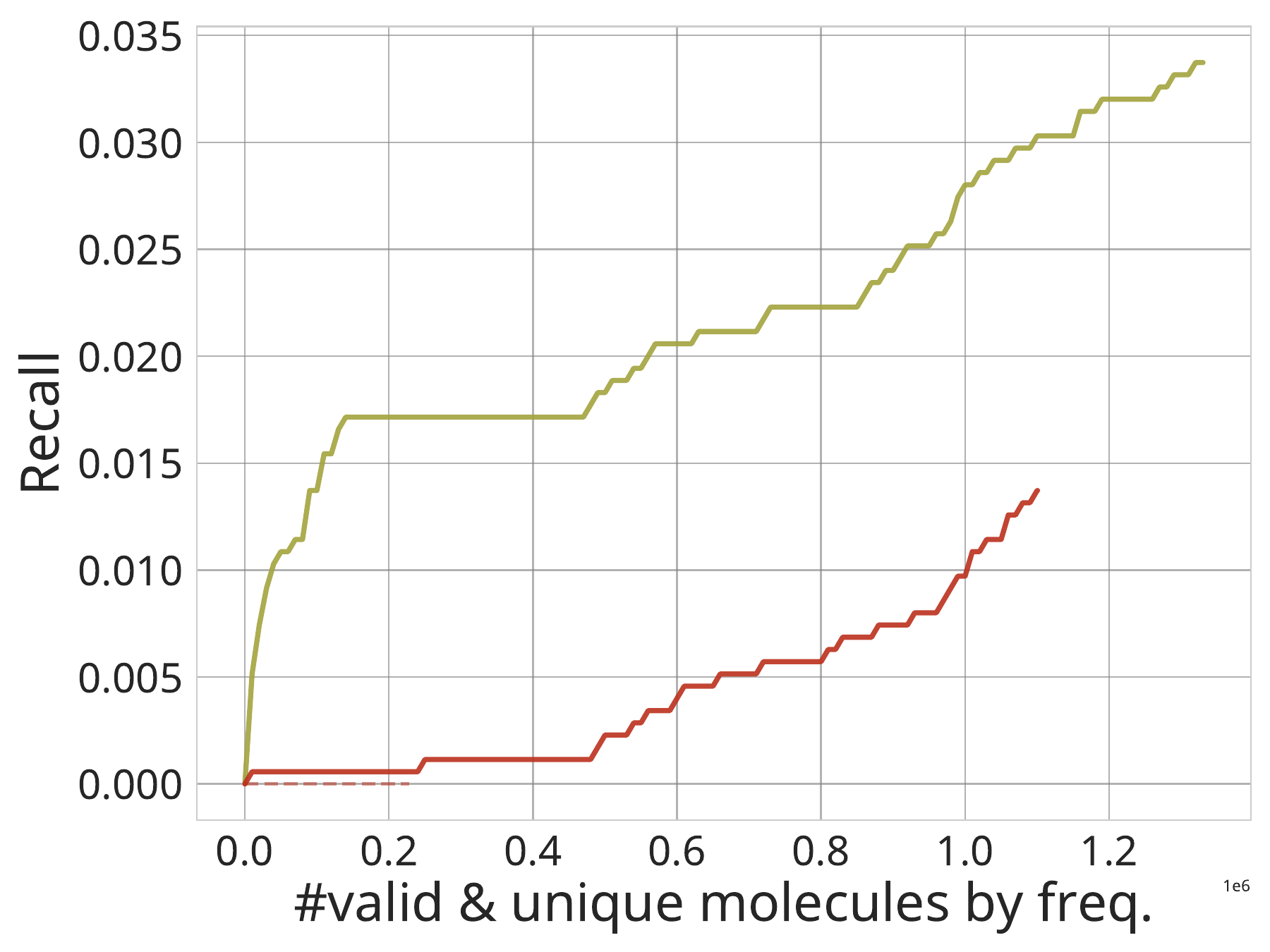}
        \caption{ZINC250k 20}
    \end{subfigure}
    \caption{Recall curves for different sizes on the dataset ZINC250k.}
    \label{fig:metrics_q3-ZINCk}
\end{figure*}

\subsubsection{Q4: How does NMLN* compare to specialized models for molecules?}

We further compared NMLN* against molecule-specific generative methods, which leverage domain-specific representations such as SMILES or SELFIES. These representations inherently encode chemical rules, enabling such models to achieve higher validity and more efficient exploration of molecular space. In Appendix Figure~\ref{fig:metrics_q4-chembl}, ~\ref{fig:metrics_q4-MOSES}, ~\ref{fig:metrics_q4-QM}, and ~\ref{fig:metrics_q4-ZINCk} we show that NMLN* outperforms molecule-specific baselines on smaller molecule sizes (9 and 10 heavy atoms), which shows their effectiveness even without relying on handcrafted molecular priors. However, at sizes 15 and 20, where both NMLN* and DiGresses begin to exhibit performance degradation, molecule-specific methods continue to cover a substantial portion of the test set. This suggests that the inductive biases encoded in molecular representations become increasingly beneficial as molecular complexity grows.

\begin{figure*}[t]
    \centering

    \vspace{0.2cm}
    \begin{subfigure}[b]{0.9\textwidth}  
        \centering
        \includegraphics[width=0.7\textwidth]{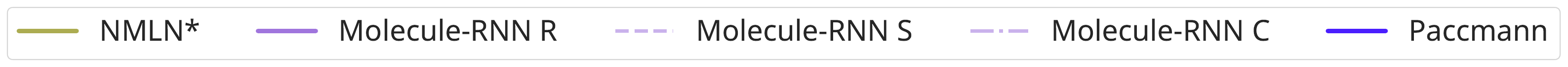}
    \end{subfigure}
    \vspace{0.2cm}
    \begin{subfigure}[b]{0.4\textwidth}
        \centering
        \includegraphics[width=\columnwidth]{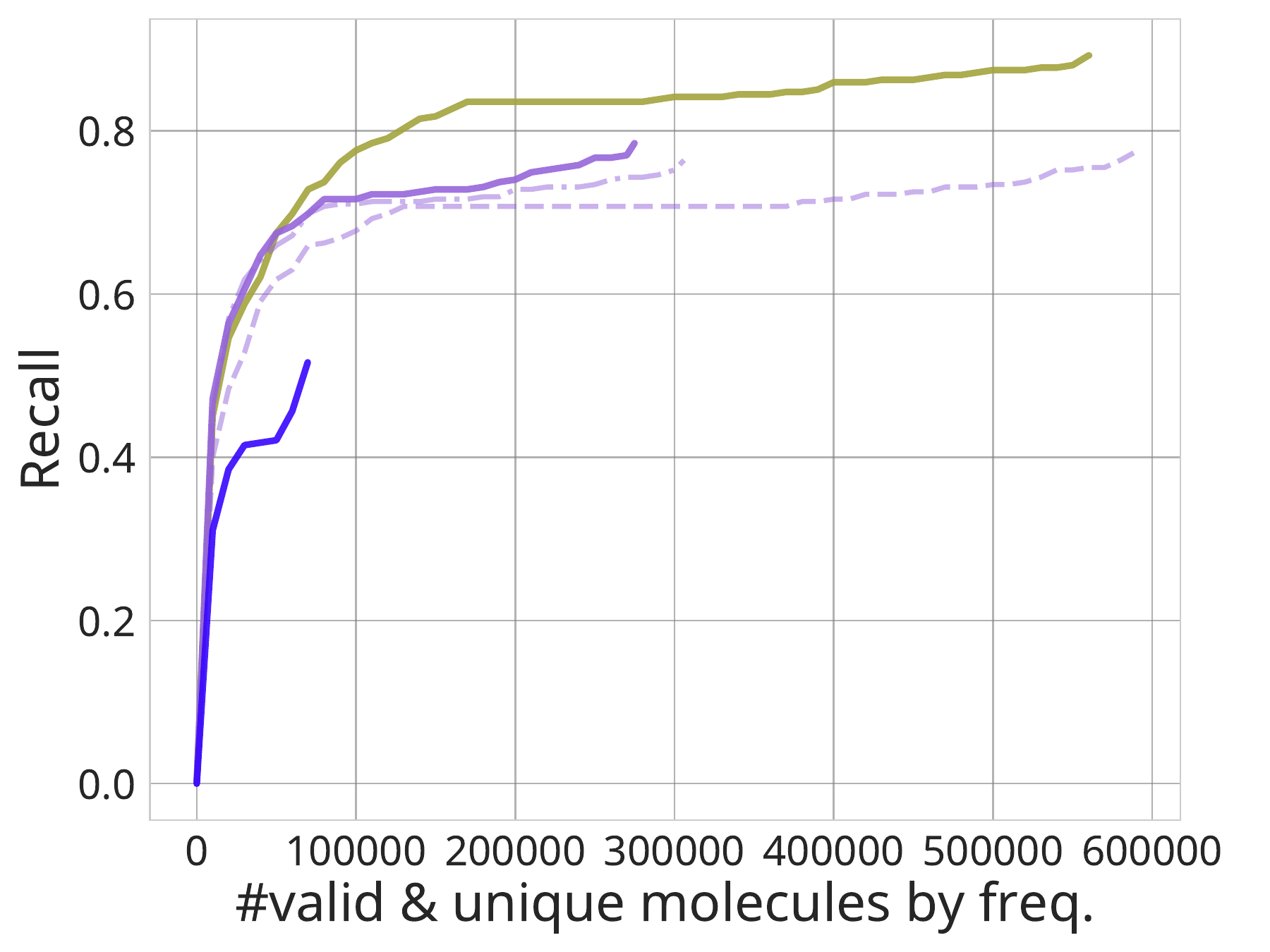}
        \caption{ChEMBL 9}
    \end{subfigure}
    \hfill
    \begin{subfigure}[b]{0.4\textwidth}
        \centering
        \includegraphics[width=\columnwidth]{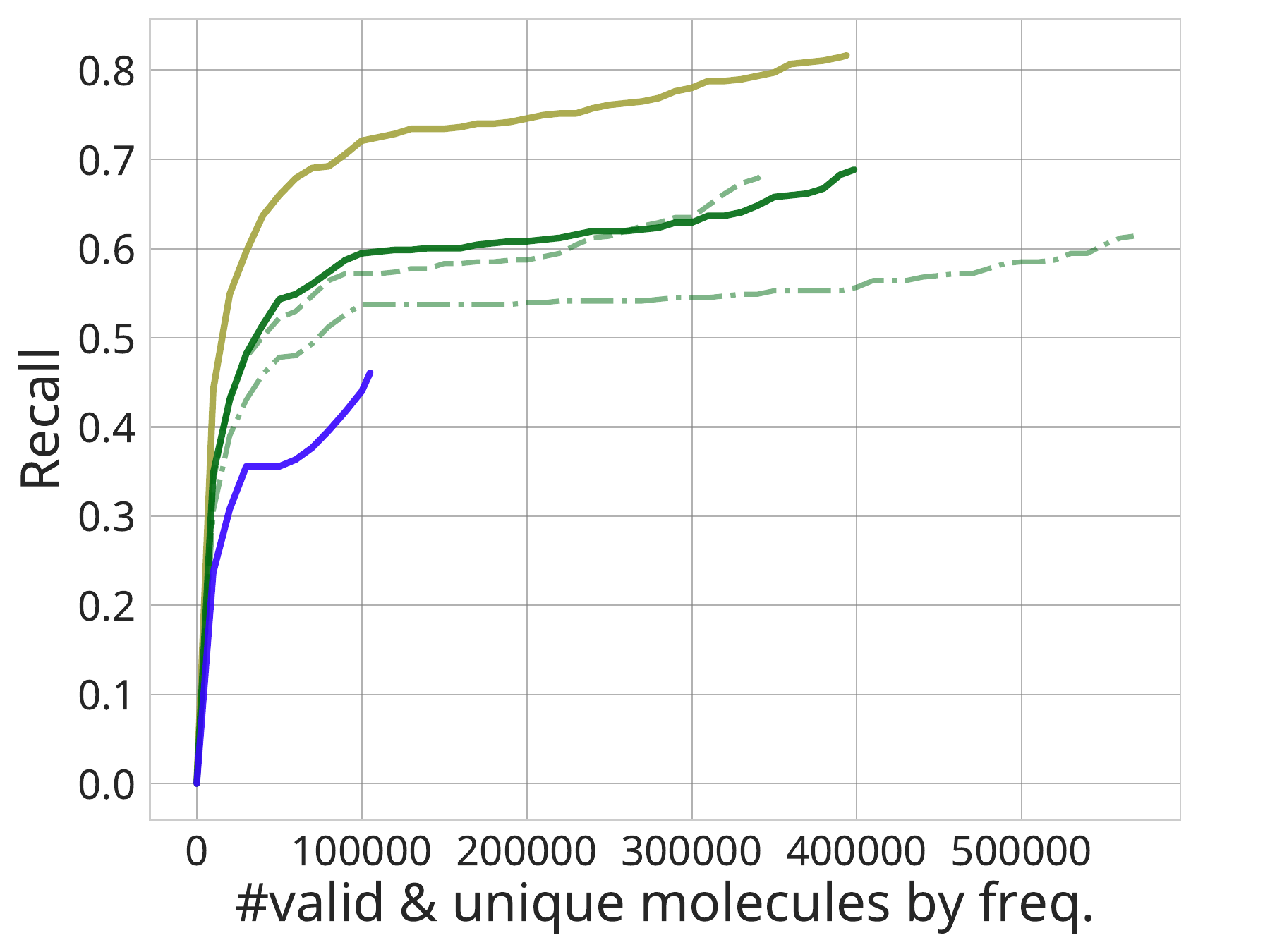}
        \caption{ChEMBL 10}
    \end{subfigure}
    \hfill
    \begin{subfigure}[b]{0.4\textwidth}
        \centering
        \includegraphics[width=\columnwidth]{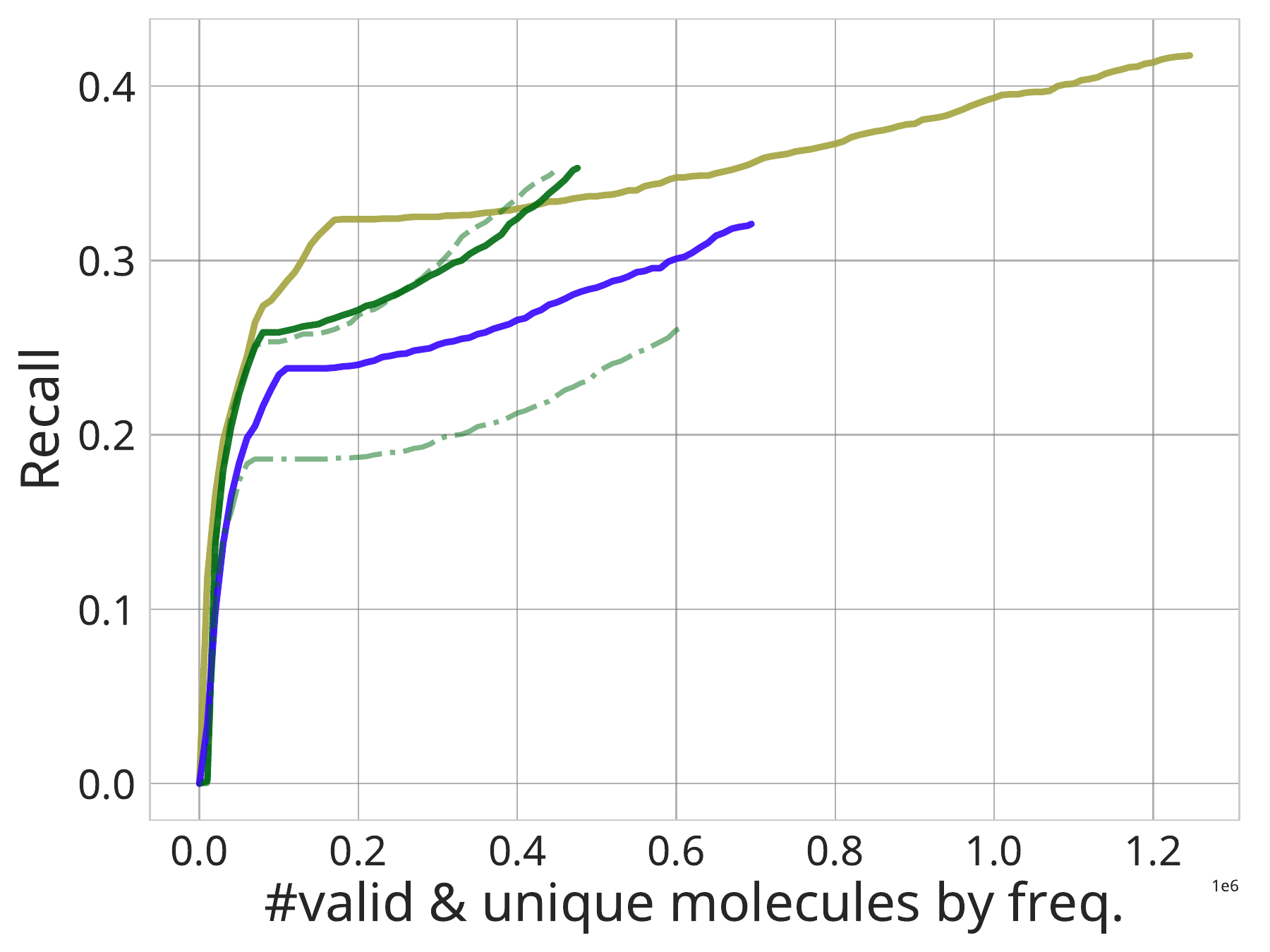}
        \caption{ChEMBL 15}
    \end{subfigure}
    \hfill
    \begin{subfigure}[b]{0.4\textwidth}
        \centering
        \includegraphics[width=\columnwidth]{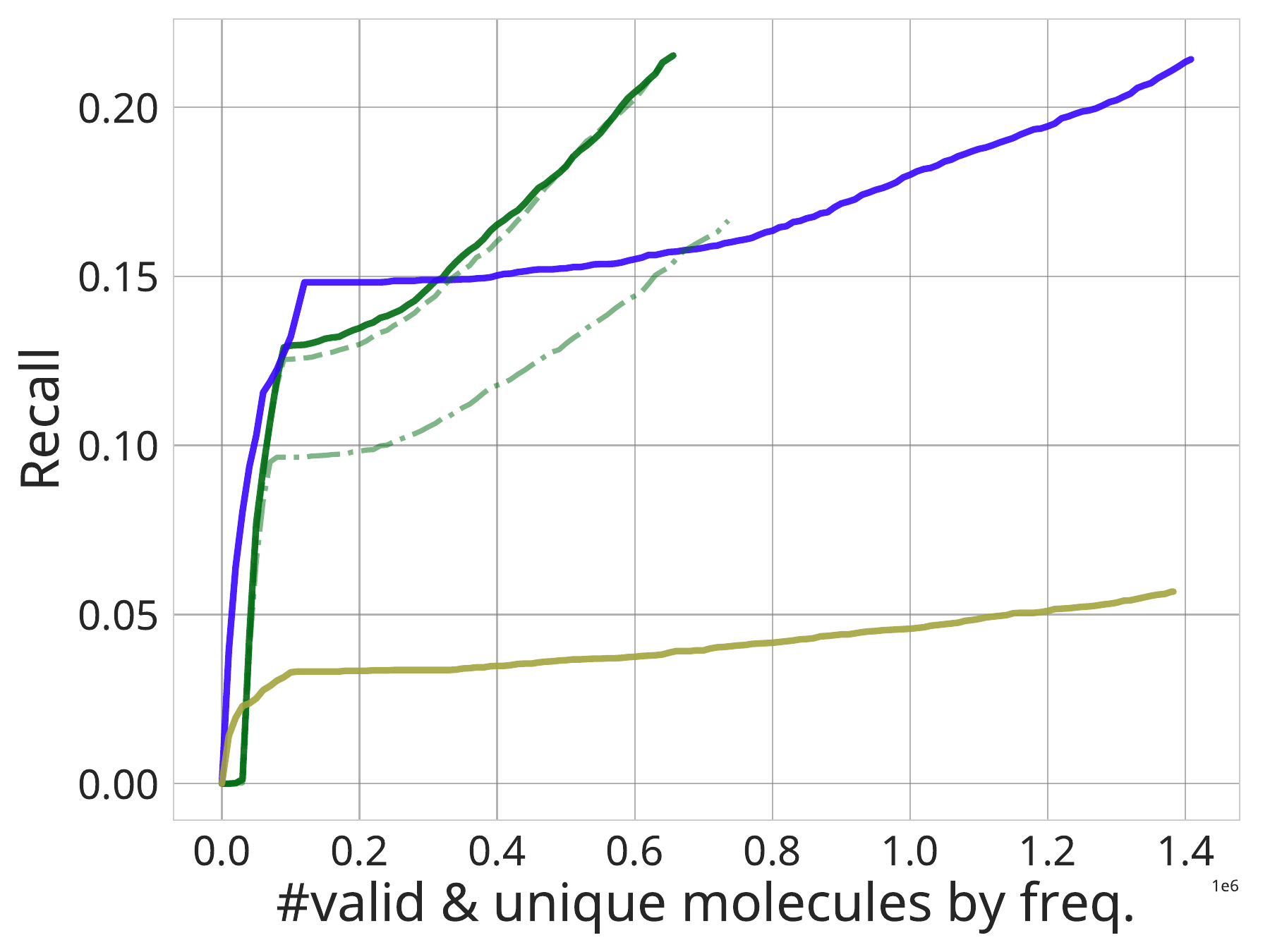}
        \caption{ChEMBL 20}
    \end{subfigure}

    \caption{Recall curves for different sizes on the dataset ChEMBL.}
    \label{fig:metrics_q4-chembl}
\end{figure*}

\begin{figure*}[t]
    \centering
    \begin{subfigure}[b]{0.9\textwidth}  
        \centering
        \includegraphics[width=0.7\textwidth]{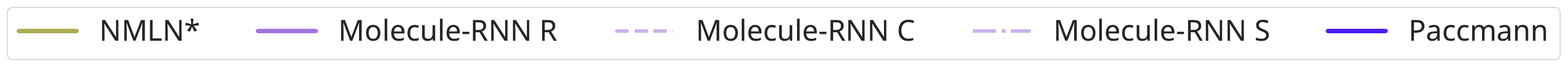}
    \end{subfigure}
    \vspace{0.2cm}
    \begin{subfigure}[b]{0.4\textwidth}
        \centering
        \includegraphics[width=\columnwidth]{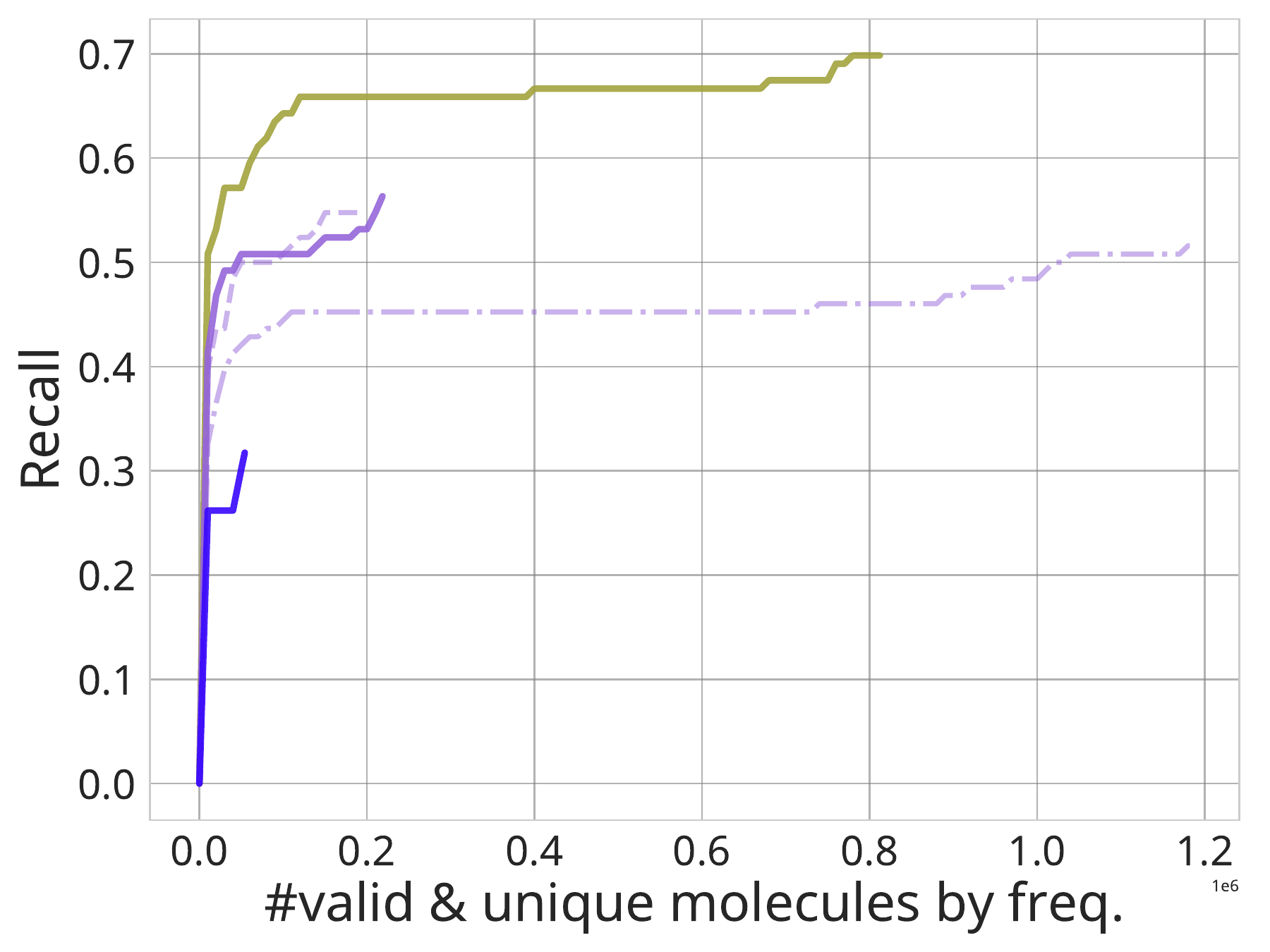}
        \caption{MOSES15}
    \end{subfigure}
    \begin{subfigure}[b]{0.4\textwidth}
        \centering
        \includegraphics[width=\columnwidth]{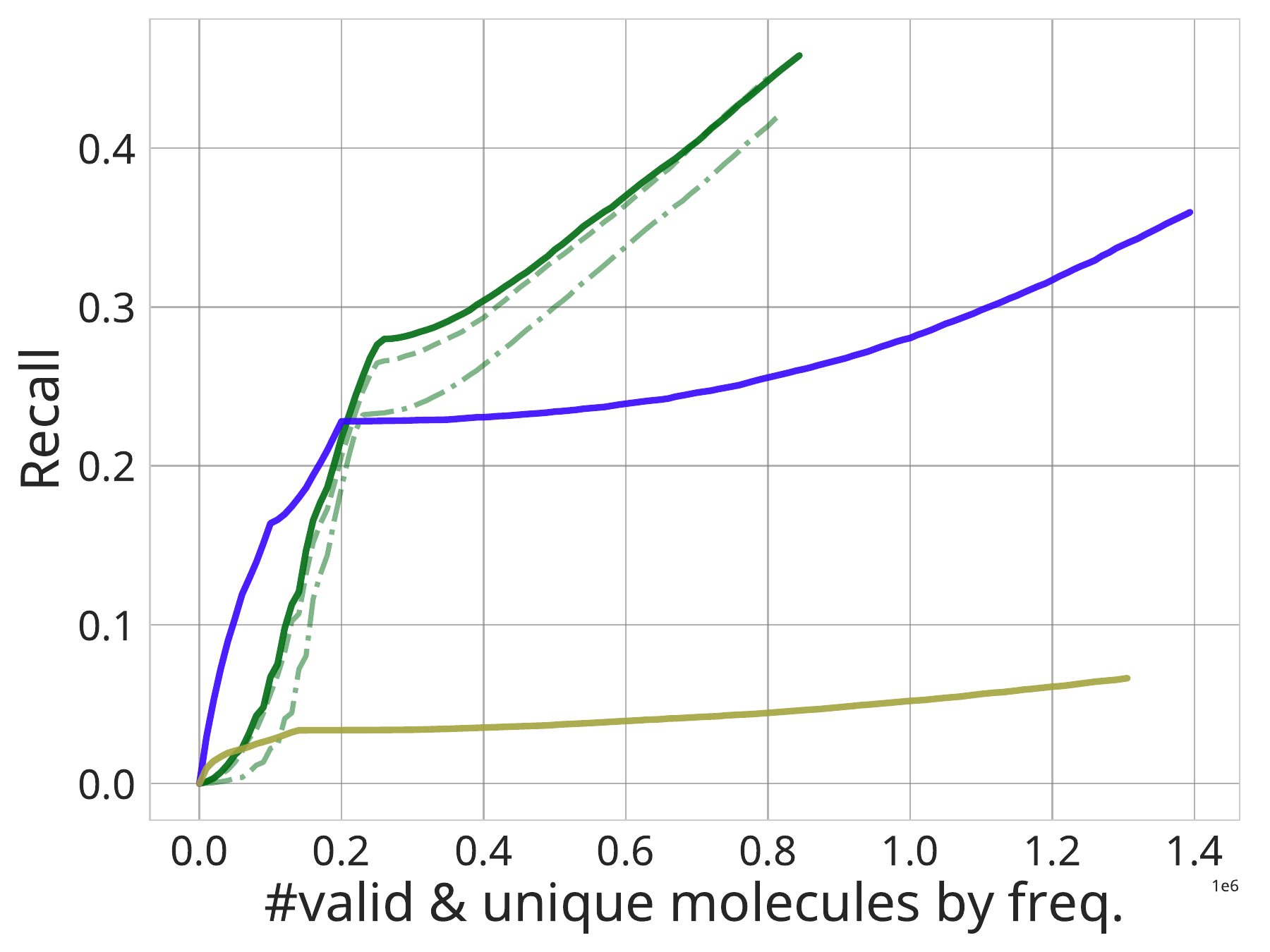}
        \caption{MOSES20}
    \end{subfigure}
    \caption{Recall curves for different sizes on the dataset MOSES.}
    \label{fig:metrics_q4-MOSES}
\end{figure*}

\begin{figure*}[t]
    \centering
    \includegraphics[width=0.4\textwidth]{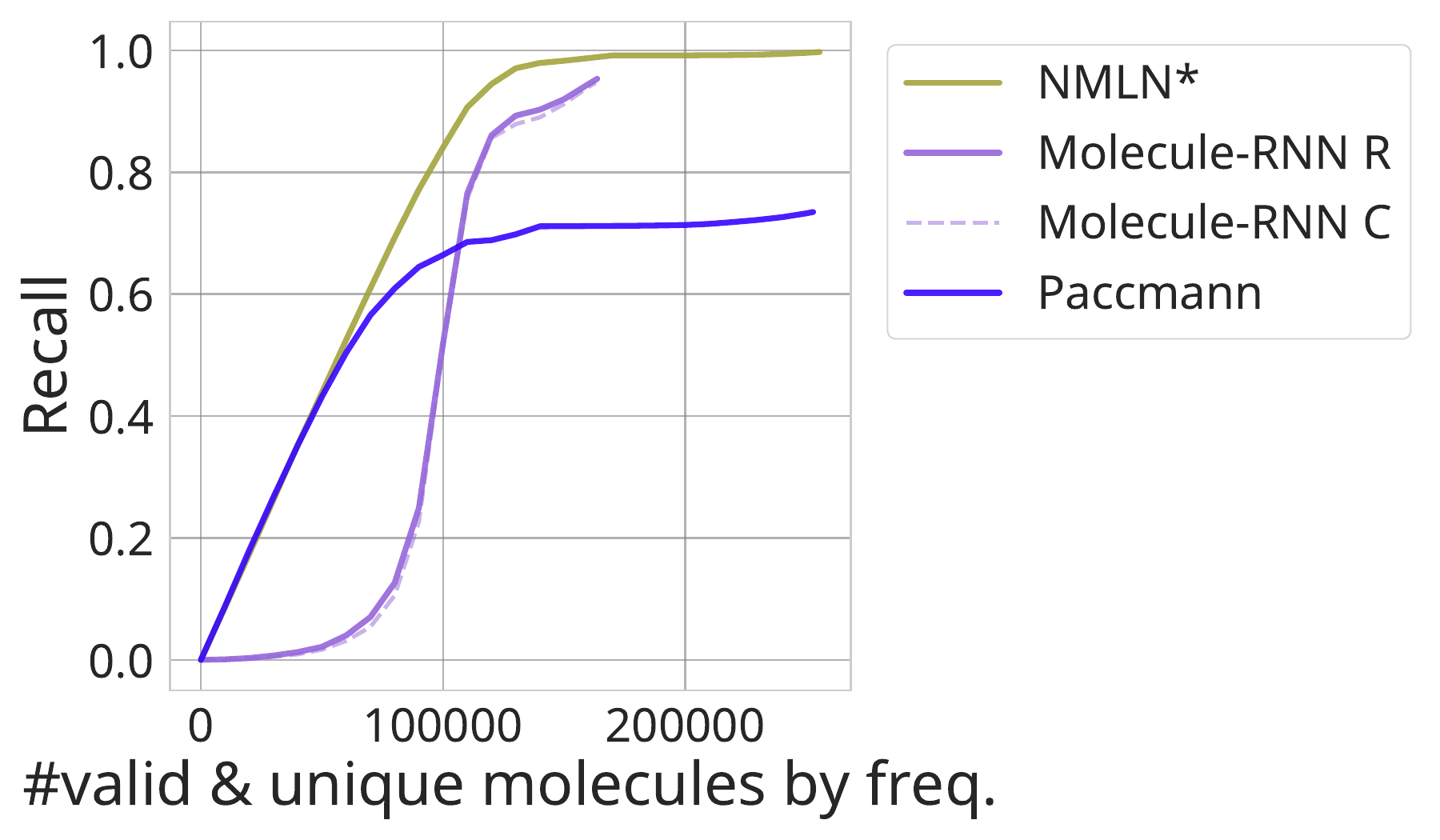}
    \caption{Recall curves for the QM9.}
    \label{fig:metrics_q4-QM}
\end{figure*}

\begin{figure*}[t]
    \centering

    \begin{subfigure}[b]{0.9\textwidth}  
        \centering
        \includegraphics[width=0.7\textwidth]{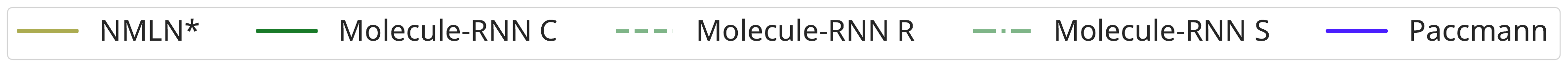}
    \end{subfigure}
    \vspace{0.2cm}

    \begin{subfigure}[b]{0.4\textwidth}
        \centering
        \includegraphics[width=\columnwidth]{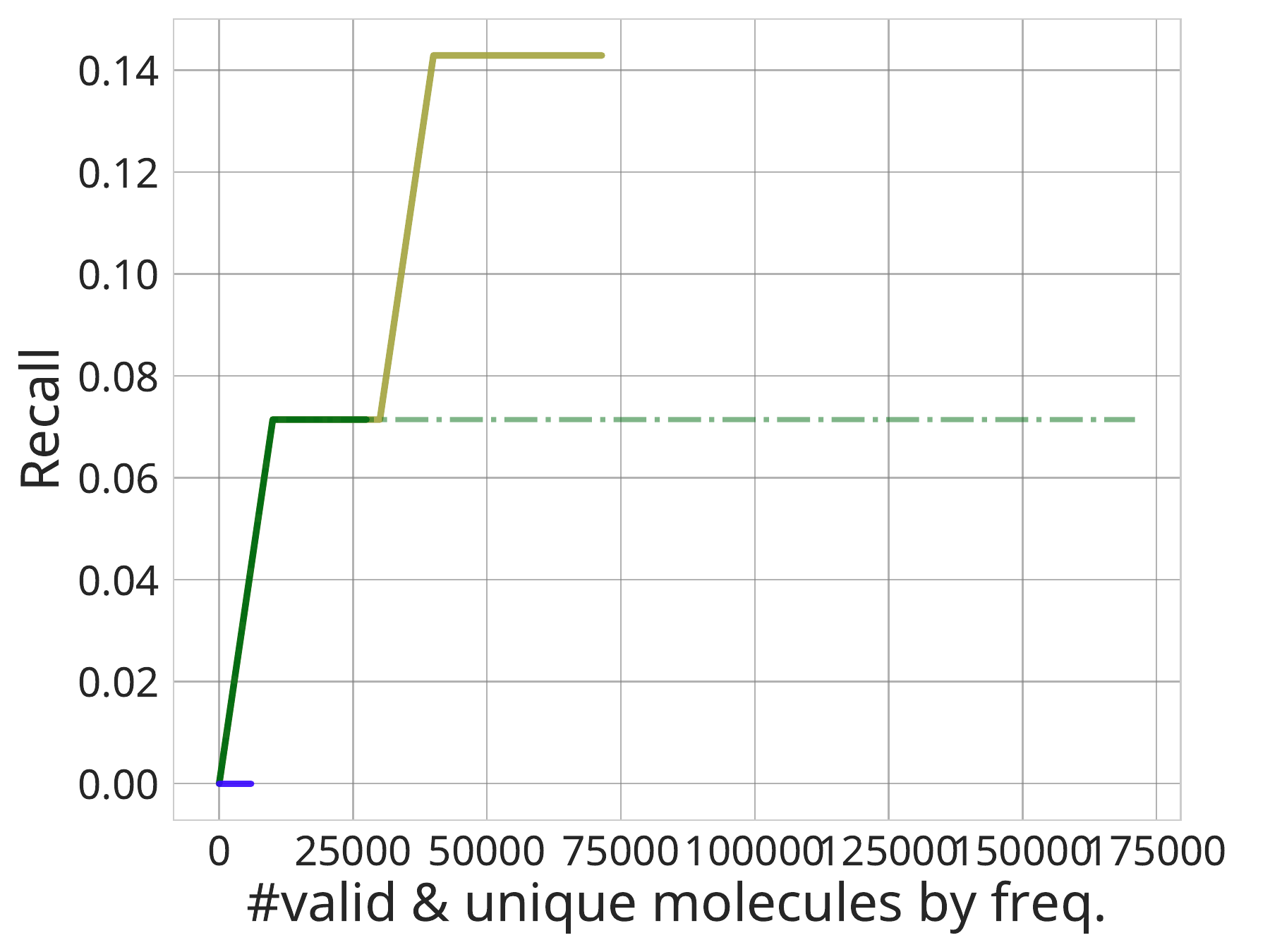}
        \caption{ZINC250k 9}
    \end{subfigure}
    \begin{subfigure}[b]{0.4\textwidth}
        \centering
        \includegraphics[width=\columnwidth]{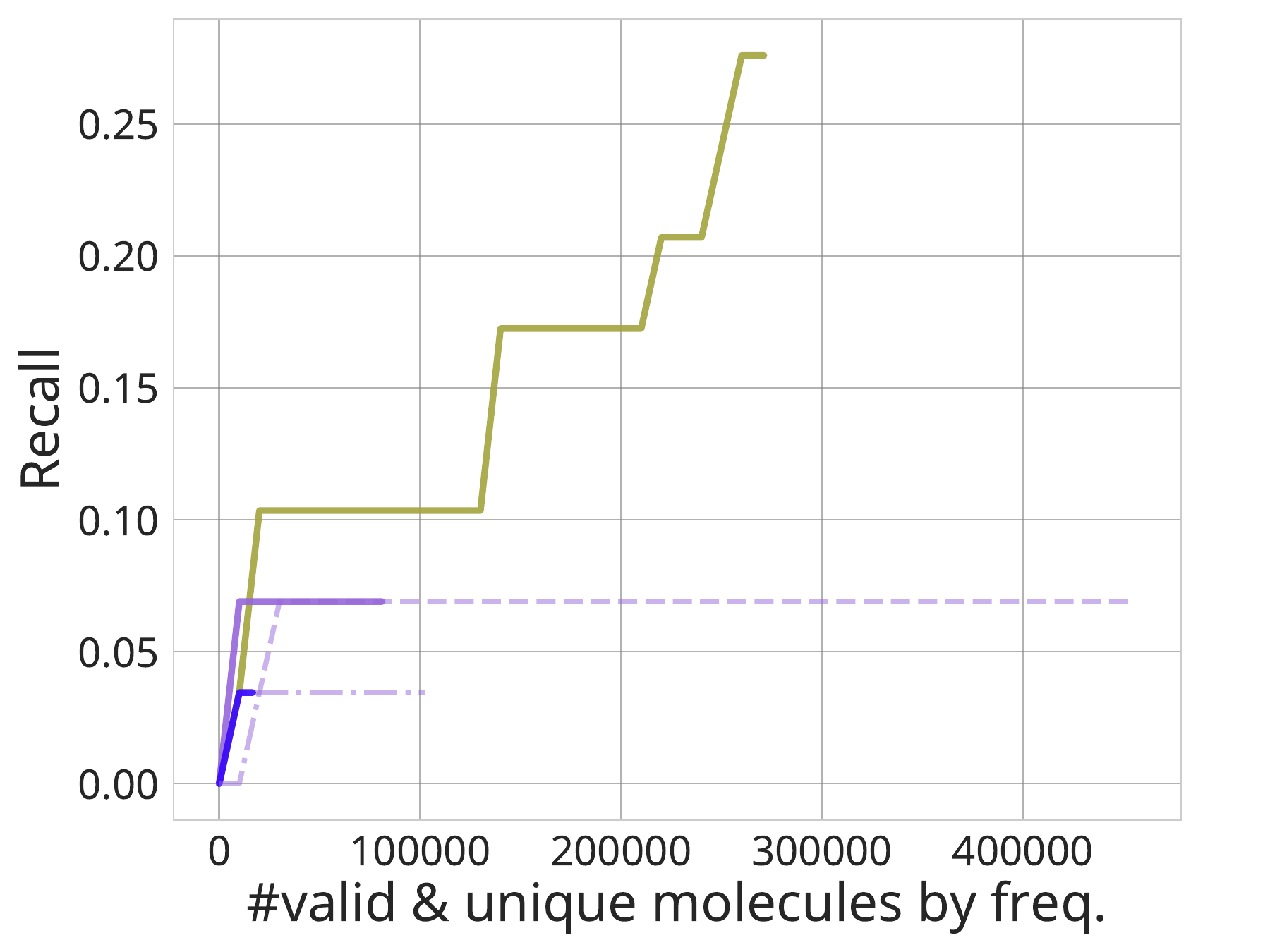}
        \caption{ZINC250k 10}
    \end{subfigure}
    \\
    \begin{subfigure}[b]{0.4\textwidth}
        \centering
        \includegraphics[width=\columnwidth]{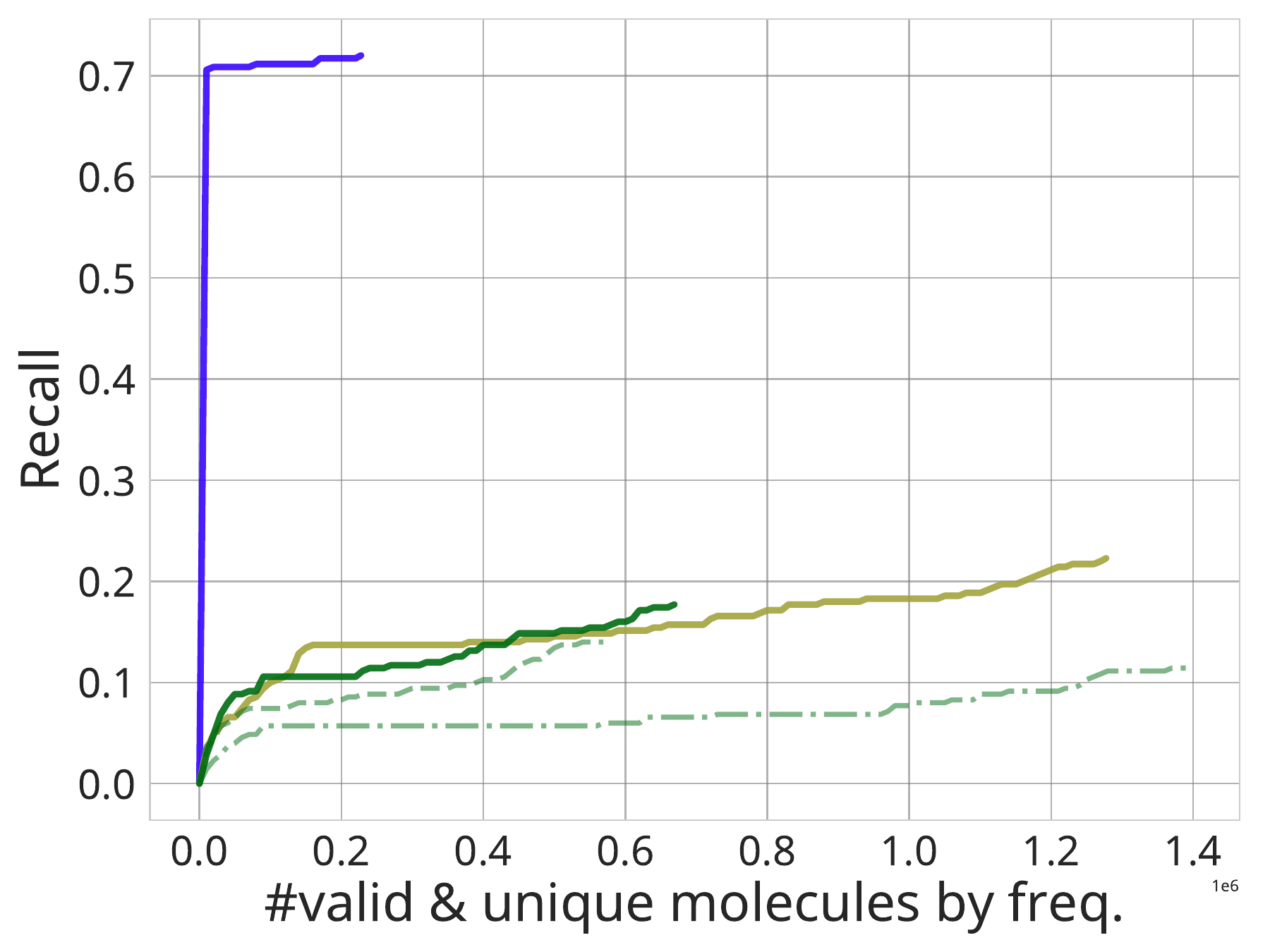}
        \caption{ZINC250k 15}
    \end{subfigure}
    \begin{subfigure}[b]{0.4\textwidth}
        \centering
        \includegraphics[width=\columnwidth]{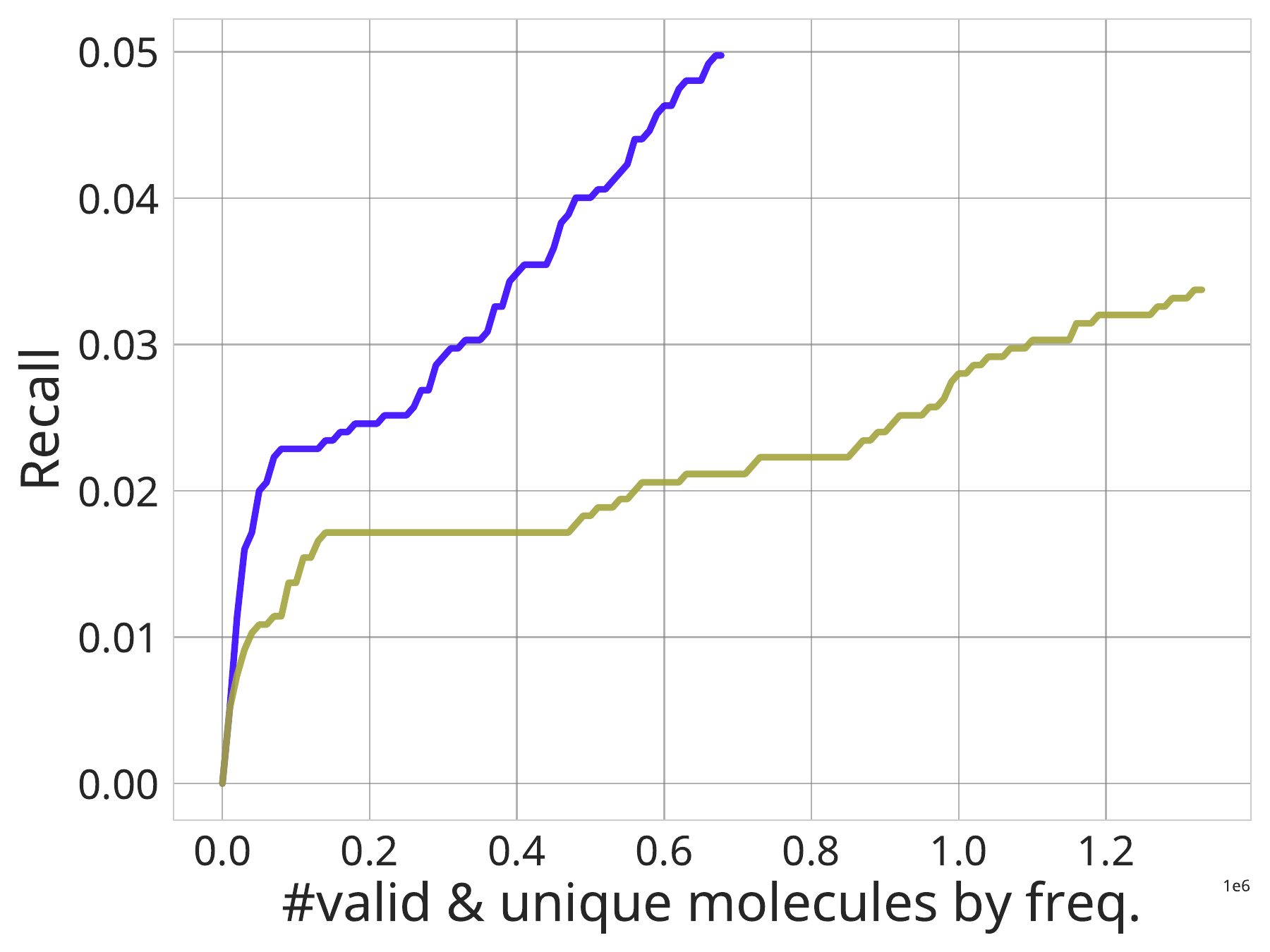}
        \caption{ZINC250k 20}
    \end{subfigure}
    \caption{Recall curves for different sizes on the dataset ZINC250k.}
    \label{fig:metrics_q4-ZINCk}
\end{figure*}

\subsubsection{A Peek Into the Chains of NMLNs*}

\paragraph{Validity Across Noise Levels}
Intuitively, chains in parallel noising with less noise should produce samples that are valid molecules more often than chains with more noise. We measured the fraction of valid molecules produced at each of the five NMLN* noise levels in parallel noising (Fig. \ref{fig:pn-nmln-valid-curve-levels}). Validity was near zero at the highest noise level (0.1) and increased monotonically as noise decreased, with the lowest level (0.001) yielding almost exclusively valid samples.

\paragraph{Recall Across Noise Levels} Figure \ref{fig:pn-nmln-recall-curve-levels} illustrates the recall curves of the five levels of single NMLN* from a single experiment, corresponding to progressively lower noise levels: 0.1, 0.01, 0.005, 0.0025, and 0.001. As expected, the first level (noise = 0.1) demonstrates limited recall due to its high stochasticity and exploratory behavior. In contrast, the final level (noise = 0.001) achieves the highest recall.

\begin{figure}[t]
    \centering
    \includegraphics[width=0.5\linewidth]{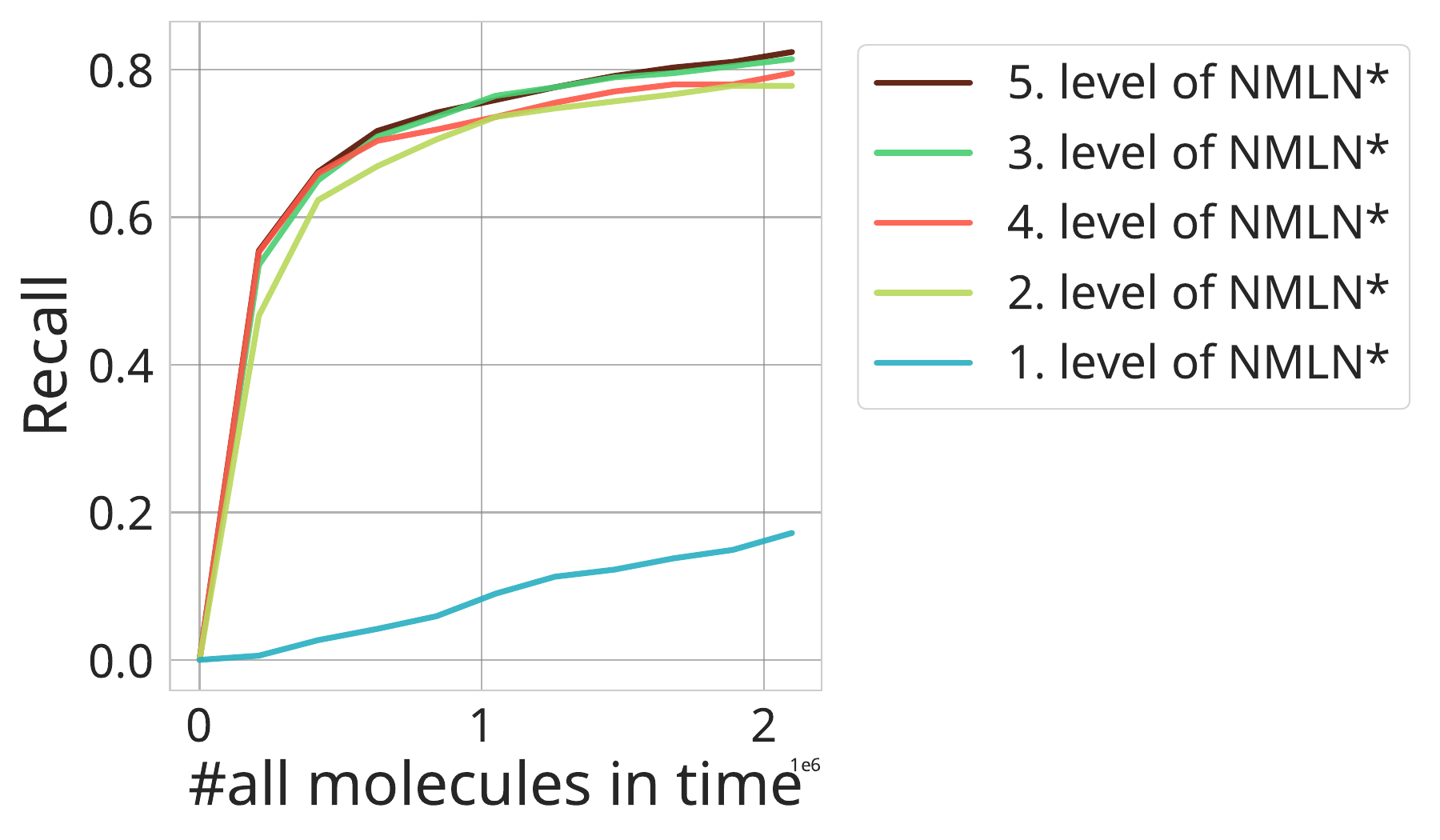}
    \caption{Recall curve of NMLN* on different levels on the dataset ChEMBL.}
    \label{fig:pn-nmln-recall-curve-levels}
\end{figure}

\begin{figure}[t]
    \centering
    \includegraphics[width=0.5\linewidth]{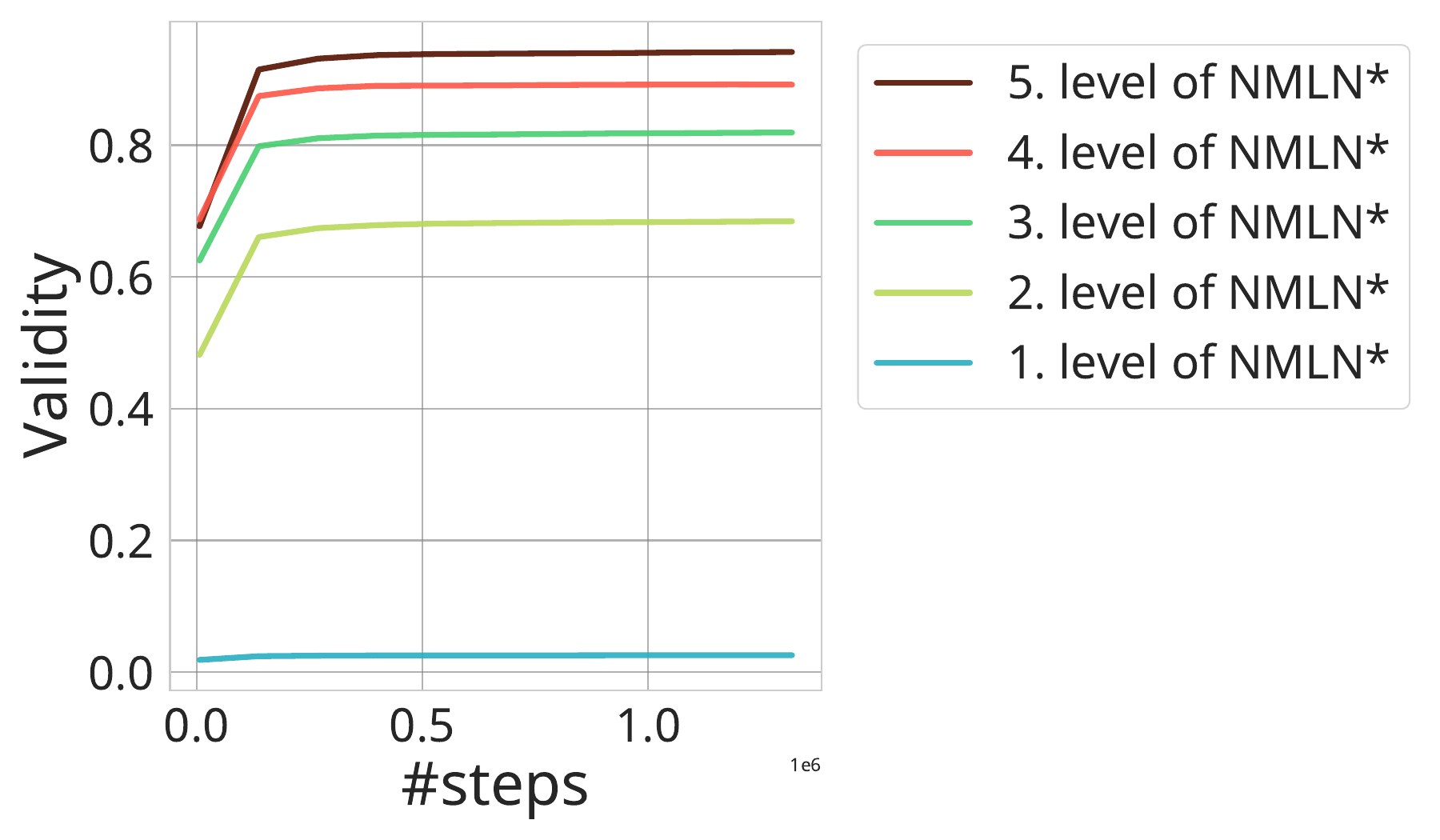}
    \caption{Validity of NMLN* on different levels on the dataset ChEMBL 10.}
    \label{fig:pn-nmln-valid-curve-levels}
\end{figure}

\section{Full Theory for Parallel Noising}\label{app:theory}
This section develops properties of parallel noising (PN) that are useful both conceptually and for designing practical ladders.
Throughout, $\Omega$ is a finite state space of possible worlds (for a fixed domain size) and each level $i$ targets $\pi_i(\omega)\propto \exp(\rho_i(\omega))$ as in Section~\ref{sec:pn}.

\subsection{Swap acceptance and distributional overlap}\label{sec:overlap}

Let $\mu=\pi_i\otimes\pi_{i+1}$ and $\mu^{\mathrm{swap}}=\pi_{i+1}\otimes\pi_i$ be distributions on $\Omega^2$.
Denote total variation distance by $\|\cdot\|_{\mathrm{TV}}$.

\begin{corollary}[A simple lower bound]\label{cor:tv-marginal}
\[
\mathbb{E}\big[\alpha_{i,i+1}(X,Y)\big]\ge 1-2\|\pi_i-\pi_{i+1}\|_{\mathrm{TV}}.
\]
\end{corollary}

\begin{proof}
Write $\delta=\pi_i-\pi_{i+1}$. Then
\begin{align*}
\mu-\mu^{\mathrm{swap}}
&=\pi_i\otimes\pi_{i+1}-\pi_{i+1}\otimes\pi_i\\
&=\delta\otimes\pi_{i+1}-\pi_{i+1}\otimes\delta.
\end{align*}
Using $\|A\|_{\mathrm{TV}}=\tfrac12\|A\|_1$, the triangle inequality, and
$\|\delta\otimes\pi\|_1=\|\delta\|_1\|\pi\|_1=\|\delta\|_1$,
\begin{align*}
\|\mu-\mu^{\mathrm{swap}}\|_{\mathrm{TV}}
&\le \tfrac12(\|\delta\|_1+\|\delta\|_1)\\
&=\|\delta\|_1\\
&=2\|\pi_i-\pi_{i+1}\|_{\mathrm{TV}}.
\end{align*}
Apply Proposition~\ref{prop:swap-overlap}.
\end{proof}

\subsection{A distribution-free bound for Bernoulli noising}

We now specialize to the Bernoulli corruption operator $K_\nu$ from Section~\ref{sec:noising}.
For any distribution $p$ on $\{0,1\}^d$, let $p_\nu=pK_\nu$ denote the $\nu$-noised distribution.

\begin{proposition}[Total variation between adjacent noise levels]\label{prop:tv-noise}
For any $p$ and any $\nu,\nu'\in[0,1]$,
\[
\|p_\nu-p_{\nu'}\|_{\mathrm{TV}}\le d\,|\nu-\nu'|.
\]
\end{proposition}

\paragraph{Interpretation.}
Proposition~\ref{prop:tv-noise} is a \emph{distribution-free} (worst-case) sanity check: it depends on the ambient dimension $d$ of the ground-atom encoding and can therefore be extremely loose for realistic relational or molecular representations (where $d$ may be in the thousands or more).
In particular, the implied acceptance lower bound in Corollary~\ref{cor:acc-lb} becomes vacuous once $d|\nu_i-\nu_{i+1}|$ is not small.
We therefore \emph{do not} use this bound to tune noise schedules.

Instead, in practice we select and validate ladders using the overlap/acceptance characterization from Section~\ref{sec:overlap}:
we monitor adjacent swap acceptance rates during training and choose $\{\nu_i\}$ such that adjacent pairs maintain non-trivial acceptance.
We also report recall/quality as a function of compute (number of energy evaluations) when comparing different ladders or numbers of replicas (Section~\ref{sec:experiments}).

\begin{proof}
Construct a coupling $(Y,Y')$ with $Y\sim p_\nu$ and $Y'\sim p_{\nu'}$ as follows:
sample $X\sim p$, sample i.i.d.\ $U_1,\dots,U_d\sim \mathrm{Unif}(0,1)$, set $B_j=\mathbf{1}\{U_j<\nu\}$ and $B'_j=\mathbf{1}\{U_j<\nu'\}$, and define $Y_j=X_j\oplus B_j$ and $Y'_j=X_j\oplus B'_j$.
Then $\mathbb{P}(Y\neq Y')\le \sum_{j=1}^d\mathbb{P}(B_j\neq B'_j)=d|\nu-\nu'|$.
By the coupling characterization of total variation,
$\|p_\nu-p_{\nu'}\|_{\mathrm{TV}}\le \mathbb{P}(Y\neq Y')\le d|\nu-\nu'|$.
\end{proof}

\begin{corollary}[Lower bound on expected swap acceptance]\label{cor:acc-lb}
Assume (idealized) that two adjacent PN levels target $\pi_i=p_{\nu_i}$ and $\pi_{i+1}=p_{\nu_{i+1}}$ for a common underlying distribution $p$.
Then in stationarity,
\[
\mathbb{E}\big[\alpha_{i,i+1}(X,Y)\big]\ge 1-2d|\nu_i-\nu_{i+1}|.
\]

\noindent This worst-case bound is informative only when $2d|\nu_i-\nu_{i+1}|\ll 1$; otherwise the right-hand side becomes vacuous (and may be negative).
\end{corollary}

\begin{proof}
Combine Corollary~\ref{cor:tv-marginal} with Proposition~\ref{prop:tv-noise}.
\end{proof}

\subsection{Invariance under potential shifts}

\begin{proposition}[Shift-invariance of swap decisions]\label{prop:shift}
Replacing $\rho_i$ by $\tilde\rho_i=\rho_i+c_i$ for any constants $c_i\in\mathbb{R}$ leaves the swap acceptance probability \eqref{eq:pn-swap} unchanged.
\end{proposition}

\begin{proof}
In $\Delta_{i,i+1}$, each $\rho_i$ appears once with a plus and once with a minus sign, so additive constants cancel.
\end{proof}

\section{Proofs}\label{app:proofs}
\paragraph{Proof for \Cref{prop:pn-correctness}.}
\begin{proof}
Let $\Omega$ be the (discrete) single-replica state space and $\Omega^N$ the joint state space.
For each noise level $j\in\{1,\dots,N\}$, let $\pi_j:\Omega\to[0,1]$ be the target pmf and
$K_j:\Omega\times\Omega\to[0,1]$ a Markov kernel satisfying invariance
$\sum_{\tilde\omega\in\Omega}\pi_j(\tilde\omega)\,K_j(\tilde\omega,\omega)=\pi_j(\omega)$ for all $\omega\in\Omega$.
All distributions and kernels are over a discrete state space, so we work with probability mass functions and sums.
Let
\[
\Pi(\omega_1,\dots,\omega_N)\;=\;\prod_{j=1}^N \pi_j(\omega_j)
\]
and
\[
K(\tilde\omega,\omega)\;=\;\prod_{j=1}^N K_j(\tilde\omega_j,\omega_j).
\]

\textbf{(1) Product updates.}
For any $\omega=(\omega_1,\dots,\omega_N)$,
\begin{align*}
(\Pi K)(\omega)
&=\sum_{\tilde\omega\in\Omega^N}\Pi(\tilde\omega)\,K(\tilde\omega,\omega)
\\
&=\sum_{\tilde\omega_1,\dots,\tilde\omega_N}\;\prod_{j=1}^N \pi_j(\tilde\omega_j)\,K_j(\tilde\omega_j,\omega_j) \\
&=\prod_{j=1}^N\left(\sum_{\tilde\omega_j\in\Omega}\pi_j(\tilde\omega_j)\,K_j(\tilde\omega_j,\omega_j)\right) \\
&=\prod_{j=1}^N \pi_j(\omega_j)
=\Pi(\omega),
\end{align*}
where we used the assumed invariance of each coordinate kernel,
$\sum_{\tilde\omega_j}\pi_j(\tilde\omega_j)K_j(\tilde\omega_j,\omega_j)=\pi_j(\omega_j)$, and applied repeated distributivity, which allowed us to change the order of summation and multiplication.
Thus $K$ leaves $\Pi$ invariant.

\textbf{(2) Swap updates.}
Fix $i$ and define the swap map $T$ on $\Omega^N$ by exchanging coordinates $i$ and $i{+}1$:
$T(\omega_1,\dots,\omega_i,\omega_{i+1},\dots,\omega_N)=(\omega_1,\dots,\omega_{i+1},\omega_i,\dots,\omega_N)$.
For any state $x\in\Omega^N$,
\[
\frac{\Pi(T(x))}{\Pi(x)}
=\frac{\pi_i(x_{i+1})\,\pi_{i+1}(x_i)}{\pi_i(x_i)\,\pi_{i+1}(x_{i+1})},
\]
so the Metropolis acceptance probability $\alpha(x)=\min\{1,\Pi(T(x))/\Pi(x)\}$ coincides with \eqref{eq:pn-swap}.
The swap proposal is deterministic and involutive ($T(T(x))=x$), hence symmetric, and we have the pointwise identity
\[
\Pi(x)\,\alpha(x)=\min\{\Pi(x),\Pi(T(x))\}=\Pi(T(x))\,\alpha(T(x)).
\]
Since the swap kernel $S_{i,i+1}$ only transitions from $x$ to $T(x)$ (with probability $\alpha(x)$) or stays at $x$,
this equality is exactly the detailed balance condition for the pair $(x,T(x))$, implying $\Pi S_{i,i+1}=\Pi$.

\textbf{(3) Compositions.}
If $\Pi P=\Pi$ and $\Pi Q=\Pi$, then $\Pi(PQ)=(\Pi P)Q=\Pi Q=\Pi$.
Therefore any finite composition of product updates $K$ and swap updates $S_{i,i+1}$ leaves $\Pi$ invariant,
and so does the overall transition kernel used in Algorithm~\ref{alg:pn}.

\textbf{(4) Marginal correctness of the lowest-noise chain.}
By construction, the $\omega_N$-marginal of $\Pi$ is $\pi_N$.
Hence, when the joint chain is at stationarity, $\omega_N\sim \pi_N$.
\end{proof}

\paragraph{Remark on \Cref{prop:swap-overlap}.}
This expected-acceptance/overlap identity is standard in replica exchange; see \citet{kofke2002acceptance} and its erratum \citep{kofke2004erratum}, which give an exact expression for the mean exchange acceptance probability in terms of the overlap of adjacent energy distributions (Eq.~(7) of the original article).
\par

\paragraph{Proof for \Cref{cor:tv-marginal}.}
\begin{proof}
Write $\delta=\pi_i-\pi_{i+1}$. Then
\begin{align*}
\mu-\mu^{\mathrm{swap}}
&=\pi_i\otimes\pi_{i+1}-\pi_{i+1}\otimes\pi_i\\
&=\delta\otimes\pi_{i+1}-\pi_{i+1}\otimes\delta.
\end{align*}
Using $\|A\|_{\mathrm{TV}}=\tfrac12\|A\|_1$, the triangle inequality, and
$\|\delta\otimes\pi\|_1=\|\delta\|_1\|\pi\|_1=\|\delta\|_1$,
\begin{align*}
\|\mu-\mu^{\mathrm{swap}}\|_{\mathrm{TV}}
&\le \tfrac12(\|\delta\|_1+\|\delta\|_1)\\
&=\|\delta\|_1\\
&=2\|\pi_i-\pi_{i+1}\|_{\mathrm{TV}}.
\end{align*}
Apply Proposition~\ref{prop:swap-overlap}.
\end{proof}

\paragraph{Proof for \Cref{prop:tv-noise}.}
\begin{proof}
Construct a coupling $(Y,Y')$ with $Y\sim p_\nu$ and $Y'\sim p_{\nu'}$ as follows:
sample $X\sim p$, sample i.i.d.\ $U_1,\dots,U_d\sim \mathrm{Unif}(0,1)$, set $B_j=\mathbf{1}\{U_j<\nu\}$ and $B'_j=\mathbf{1}\{U_j<\nu'\}$, and define $Y_j=X_j\oplus B_j$ and $Y'_j=X_j\oplus B'_j$.
Then $\mathbb{P}(Y\neq Y')\le \sum_{j=1}^d\mathbb{P}(B_j\neq B'_j)=d|\nu-\nu'|$.
By the coupling characterization of total variation,
$\|p_\nu-p_{\nu'}\|_{\mathrm{TV}}\le \mathbb{P}(Y\neq Y')\le d|\nu-\nu'|$.
\end{proof}

\paragraph{Proof for \Cref{cor:acc-lb}.}
\begin{proof}
Combine Corollary~\ref{cor:tv-marginal} with Proposition~\ref{prop:tv-noise}.
\end{proof}

\paragraph{Proof for \Cref{prop:shift}.}
\begin{proof}
In $\Delta_{i,i+1}$, each $\rho_i$ appears once with a plus and once with a minus sign, so additive constants cancel.
\end{proof}

\end{document}